\documentclass{article}

\usepackage[preprint]{neurips_2023}

\usepackage{natbib}
\setcitestyle{numbers,square}

\usepackage{blindtext}
 
\usepackage{xcolor}
\usepackage[utf8]{inputenc} 
\usepackage{mysymbol}

\usepackage[T1]{fontenc}    
\usepackage{hyperref}       
\usepackage{url}            
\usepackage{booktabs}   
\usepackage{color}
\usepackage{amsfonts}       
\usepackage{nicefrac}       
\usepackage{microtype}      
\usepackage{lipsum}
\usepackage{fancyhdr}       
\usepackage{graphicx}       
\usepackage{enumitem}

\usepackage[utf8]{inputenc} 
\usepackage[T1]{fontenc}    
\usepackage{url}            
\usepackage{booktabs}       
\usepackage{amsfonts}       
\usepackage{nicefrac}       
\usepackage{microtype}      

\usepackage{algorithm}
\usepackage{longtable}

\usepackage{amsthm}
\usepackage{dsfont}

\usepackage{tabularray}

\theoremstyle{plain}
\newtheorem{theorem}{Theorem}[section]

\newtheorem{lemma}{Lemma}[section]
\newtheorem{cor}{Corollary}[section]

\newtheorem{mydef}{Definition}[section]
\newtheorem{assumption}{Assumption}[section]

\theoremstyle{remark}
\newtheorem*{remark}{Remark}

\newcommand{\E}{\mathbb{E}}

\newcommand{\sA}{{\mathscr A}}

\newcommand{\sD}{{\mathscr D}}

\newcommand{\sR}{{\mathscr R}}

\newcommand{\sT}{{\mathscr T}}

\newcommand{\sX}{{\mathscr X}}
\newcommand{\sY}{{\mathscr Y}}
\newcommand{\sZ}{{\mathscr Z}}

\newcommand{\DG}{\sD_G(S)}

\newcommand{\bc}{{\mathbf c}}

\newcommand{\bh}{{\mathbf h}}

\newcommand{\bV}{{\mathbf V}}
\newcommand{\bW}{{\mathbf W}}
\newcommand{\bX}{{\mathbf X}}

\newcommand{\bZ}{{\mathbf Z}}

\newcommand{\bx}{{\mathbf x}}

\newcommand{\bw}{{\mathbf w}}
\newcommand{\bz}{{\mathbf z}}

\let\P\undefined
\let\E\undefined
\newcommand{\R}{\mathbb R}
\newcommand{\E}{\mathbb E}
\newcommand{\P}{\mathbb P}

\newcommand{\Saug}{\widetilde{S}}
\newcommand{\Daug}{\widetilde{\sD}}

\newcommand\Sdrop[1]{S^{\setminus #1}}
\newcommand\Schange[1]{S^{#1}}

\usepackage[utf8]{inputenc}
\usepackage{pgfplots}
\usepackage{graphicx}
\usepackage{subcaption}
\usepackage{pstricks}
\usepackage{amsmath}

\allowdisplaybreaks

\usepackage{tcolorbox}
\usepackage{tikz}
\usepackage{pgfplots}
\usepackage{wrapfig}
\usepackage{lipsum}  
\usetikzlibrary{positioning}
\usepackage{tcolorbox}
\usepackage{amsfonts,amssymb}
\usepackage{mathtools}
\usepackage{commath}
\usepackage{relsize}
\usepackage{bm}
\usepackage[font={small}]{caption}
\usepackage{comment,color,soul}
\usetikzlibrary{arrows,automata}
\usepackage{amsfonts}
\usepackage{url}
\usepackage{lipsum}
\usepackage[thinlines]{easytable}
\usepackage{nicefrac}
\usepackage{algorithm}
\usepackage{algpseudocode}
\usepackage{float}
\usepackage{multirow}
\usepackage{colortbl}
\usepackage[title]{appendix}

\usepackage[mathscr]{euscript}
\usepackage{bbm}

\DeclarePairedDelimiter{\tri}{\langle}{\rangle}

\title{Toward Understanding Generative Data Augmentation}

\author{%
  Chenyu Zheng$^{1,2}$, Guoqiang Wu$^3$, Chongxuan Li$^{1,2}$\thanks{Corresponding author.} \\
  $^1$ Gaoling School of Artificial Intelligence, Renmin University of China, Beijing, China\\
  $^2$ Beijing Key Laboratory of Big Data Management and Analysis Methods, Beijing, China\\
  $^3$ School of Software, Shandong University, Shandong, China \\
  \texttt{\{chenyu.zheng666, guoqiangwu90\}@gmail.com;}
  \texttt{chongxuanli@ruc.edu.cn}
}

\begin{document}

\maketitle

\setcounter{tocdepth}{-1}

\begin{abstract}
Generative data augmentation, which scales datasets by obtaining fake labeled examples from a trained conditional generative model, boosts classification performance in various learning tasks including (semi-)supervised learning, few-shot learning, and adversarially robust learning. However, little work has theoretically investigated the effect of generative data augmentation. To fill this gap, we establish a general stability bound in this not independently and identically
distributed (non-i.i.d.) setting, where the learned distribution is dependent on the original train set and generally not the same as the true distribution. Our theoretical result includes the divergence between the learned distribution and the true distribution. It shows that \textit{generative data augmentation can enjoy a faster learning rate when the order of divergence term is $o(\max\left( \log(m)\beta_m, 1 / \sqrt{m})\right)$}, where $m$ is the train set size and $\beta_m$ is the corresponding stability constant. We further specify the learning setup to the Gaussian mixture model and generative adversarial nets. We prove that \textit{in both cases, though generative data augmentation does not enjoy a faster learning rate, it can improve the learning guarantees at a constant level when the train set is small, which is significant when the awful overfitting occurs}. Simulation results on the Gaussian mixture model and empirical results on generative adversarial nets support our theoretical conclusions. Our code is available at \emph{\href{https://github.com/ML-GSAI/Understanding-GDA}{https://github.com/ML-GSAI/Understanding-GDA}}.








\end{abstract}

\section{Introduction}
\label{sec: intro}

Deep generative models~\cite{DBLP:journals/corr/vae,goodfellow2020generative,DBLP:conf/nips/HoJA20,DBLP:conf/iclr/0011SKKEP21} have achieved great success in many fields, including computer vision~\cite{ramesh2021zero,rombach2022high}, natural language processing~\cite{brown2020language,raffel2020exploring,leiter2023chatgpt}, and cross-modal learning~\cite{wang2022image,bao2023one,openai2023gpt} in the recent years. A promising usage built upon them is generative data augmentation (GDA), which scales the train set by producing synthetic examples with labels based on advanced conditional generative models. Empirically, it has been observed that GDA can improve classification performance in lots of settings, including supervised learning~\cite{DBLP:journals/corr/abs-2304-08466,DBLP:conf/icassp/BesnierJBCP20}, semi-supervised learning~\cite{DBLP:conf/nips/KingmaMRW14, DBLP:conf/nips/LiXZZ17,DBLP:journals/corr/abs-2302-10586}, few-shot learning~\cite{DBLP:journals/corr/abs-2302-07944}, zero-shot learning~\cite{DBLP:journals/corr/abs-2210-07574}, adversarial
robust learning~\cite{DBLP:journals/corr/abs-2103-01946,DBLP:journals/corr/abs-2302-04638}, etc.

Although promising algorithms and applications of GDA emerge in different learning setups, our experiments in Section~\ref{sec: exp} show that GDA does not always work, such as in the case with a rich train set or standard augmentation methods (e.g., flip). Besides, the number of augmented data has a significant impact on the performance while is often tuned manually.  These phenomena motivate us to study the effect of GDA. Unfortunately, little work has investigated this technique from a theoretical perspective.  Therefore, in this paper, we take a first step towards understanding it. Specially, we consider the supervised classification setting, and try to answer the following questions rigorously:

\begin{itemize}
    \item \textit{Can we establish learning guarantees for GDA and explain when it works precisely?}
    \item \textit{Can we obtain theoretical insights on hyperparameters like the number of augmented data?}
\end{itemize}

Our first main contribution is to propose a general algorithmic stability bound for GDA in Section~\ref{sec: General generalization bound}. The main technical challenge is that GDA breaks the primary i.i.d. assumption of the classical results~\cite{DBLP:journals/jmlr/BousquetE02,DBLP:conf/colt/BousquetKZ20} because the distribution learned by the generative model is dependent on the sampled train set and generally not the same as the true distribution. Besides, it is unclear whether the existing general non-i.i.d. stability bounds~\cite{DBLP:conf/nips/MohriR07,DBLP:journals/jmlr/MohriR10,DBLP:conf/nips/ZhangL0W19} are suitable to derive meaningful guarantees for GDA. Informally, our result (Theorem~\ref{thm: main generalization bound}) can be presented as follows:
\begin{align*}
    \vert \textit{Gen-error}  \vert \lesssim \text{distributions' divergence}\ + \ \text{generalization error w.r.t. mixed distribution},
\end{align*}
where \textit{Gen-error} means the generalization error of GDA, and $a \lesssim b$ means $a = O(b)$. The distributions’ divergence term on the right hand is caused by the divergence between the distribution learned by the generative model and the true distribution. In addition, the remaining generalization error w.r.t. mixed distribution vanishes as we increase the augmentation size. Comparing this bound to the classical result without GDA (Theorem~\ref{thm: classical stability bound}), we can obtain an exact condition for GDA to be effective: \textit{GDA can enjoy a faster learning rate when the order of divergence term is $o(\max\left( \log(m)\beta_m, 1 / \sqrt{m})\right)$}, where $m$ is the train set size and $\beta_m$ is the corresponding uniform stability constant. This means the performance of the chosen generative model matters a lot.

Our second main contribution is to particularize the general results to the binary Gaussian mixture model (bGMM) and generative adversarial nets (GANs)~\cite{goodfellow2020generative} in Section~\ref{sec: theory Results on bGMM} and Section~\ref{sec: Implications on deep generative models}, respectively. Our theoretical results (Theorems~\ref{thm: bGMM generalization bound} and~\ref{thm: GAN generalization bound}) show that, in both cases, the order of the divergence term in the obtained upper bound is $\Omega(\max\left( \log(m)\beta_m, 1 / \sqrt{m})\right)$. It suggests that: on the one hand, when the train set size is large enough, it is hopeless to use GDA to boost the classification performance by a large margin. Worse still, GDA may damage the generalization of the learning algorithm. On the other hand, \textit{when the train set size is small and awful overfitting happens, GDA can improve the learning guarantee at a constant level, which is significant in this situation}. These theoretical implications show the promise of GDA in real-world problems with limited data.

Finally, experiments presented in Section~\ref{sec: exp} validate our theoretical findings. In particular, in the bGMM setting, 
experimental results show that our generalization bound (Theorem~\ref{thm: bGMM generalization bound}) predicts the order and trend of true generalization error well. Besides, in our empirical study on the real image dataset, we find that GANs can not boost the test performance obviously and even damage the generalization when standard data augmentation methods are used to approximate the case with a large train set. In contrast, GANs improve the performance by a large margin when the train set size is small and terrible overfitting occurs. All these experimental results support our theoretical implications in Section~\ref{sec: main results}. Furthermore, we also conduct experiments with the state-of-the-art diffusion model~\cite{DBLP:journals/corr/EDM}. Empirical results show the promise of the diffusion model in GDA and suggest it could have a faster learning rate than GAN.

\section{Preliminaries}
\label{sec: Preliminaries}

\subsection{Notations}
Let $\sX \subseteq \R^d$ be the input space and $\sY \subseteq \R$ be the label space. We denote by $\sD$ the population distribution over $\sZ = \sX \times \sY$. The $L_p$ norm of a random variable $X$ is denoted as $\Vert X\Vert_p = (\E|X|^p)^{\frac{1}{p}}$. Given a set $S = \{\bz_1, \bz_2, \dots, \bz_m\}$, we define $\Sdrop{i}$ as the set after removing the $i$-th data point in the set $S$, and $\Schange{i}$ as the set after replacing the $i$-th data point with $\bz_i'$ in the set $S$. Let $[m] = \{1,2,\dots,m\}$, then for every set $V \subseteq [n]$, we define $S_V = \{\bz_i: i \in V\}$. In addition, for some function $f = f(S)$, we denote its conditional $L_p$ norm with respect to $S_V$ by $\norm{f}_p(S_V) = (\E[\ \norm{f}^p \mid S_V])^{\frac{1}{p}}$. Besides, we denote the total variation distance by $d_{\mathrm{TV}}$ and KL divergence by $d_{\mathrm{KL}}$, respectively.

We let $(\sY)^{\sX}$ be the set of all measurable functions from $\sX$ to $\sY$, $\sA$ be a learning algorithm and $\sA(S) \in (\sY)^{\sX}$ be the hypothesis learned on the dataset $S$. Given a learned hypothesis $\sA(S)$ and a loss function $\ell: (\sY)^{\sX} \times \sZ \rightarrow \R_+$, the true error $\sR_{\sD}(\sA(S))$ with respect to the data distribution $\sD$ is defined as $\E_{\bz \sim \sD} [\ell(\sA(S), \bz)]$. In addition, the corresponding empirical error $\widehat{\sR}_{S}(\sA(S))$ is defined as $\frac{1}{m} \sum_{i=1}^m \ell(\sA(S), \bz_i)$.

\subsection{Generative data augmentation}
\label{sec: Preliminaries-GA}
In this part, we describe the process of GDA in a mathematical way. Given a training set $S$ with $m_S$ i.i.d. examples from $\sD$, we can train a conditional generative model $G$, and denote the model distribution by $\DG$. We note that the randomness from training the generative model is ignored in this paper. In addition, we define the expectation of the model distribution with regard to $S$ as $\sD_G = \E_S[\DG]$. Based on the trained generative model, we can then obtain a new dataset $S_G$ with $m_G$ i.i.d. samples from $\DG$, where $m_G$ is a hyperparameter. Typically, we consider the case that $m_G = \Omega(m_S)$ if GDA is utilized. We denote the total number of the data points in augmented set $\Saug = S \cup S_G$ by $m_T$. Besides, we define the mixed distribution after augmentation as $\Daug(S) = \frac{m_S}{m_T} \sD + \frac{m_G}{m_T} \DG$. As a result, a hypothesis $\sA(\Saug)$ can be learned on the augmented dataset $\widetilde{S}$. To understand the effect of GDA, we are interested in the generalization error $\vert{\sR_{\sD}(\sA(\Saug)) - \widehat{\sR}_{\Saug}(\sA(\Saug))}\vert$ with regard to the learned hypothesis $\sA(\Saug)$. For convenience, we denote it by \textit{Gen-error} in the remaining paper. Technically, we establish bounds for \textit{Gen-error} using the algorithmic stability introduced in the next subsection. As far as we know, this is the first work to investigate the learning guarantees for GDA.

\subsection{Generalization via algorithmic stability}

Algorithmic stability analysis is a important tool to provide guarantees for the generalization of machine learning models. A key advantage of stability analysis is that it exploits particular properties of the algorithm and provides algorithm-dependent bounds. Different notations of stability have been proposed and used to establish high probability bounds for generalization error~\cite{DBLP:journals/jmlr/BousquetE02, DBLP:journals/jmlr/Shalev-ShwartzSSS10,DBLP:conf/icml/KuzborskijL18,DBLP:conf/icml/LiuLNT17}. Among them, uniform stability is the most widely used and has been utilized to analyze the generalization of many learning algorithms, including regularized empirical risk minimization (ERM) algorithms~\cite{DBLP:journals/jmlr/BousquetE02} and stochastic gradient descent (SGD)~\cite{DBLP:conf/icml/KuzborskijL18,DBLP:conf/icml/HardtRS16, DBLP:conf/uai/ZhangZBP0022}. The uniform stability is defined as the following.

\begin{mydef}[Uniform stability]
Algorithm $\sA$ is uniformly $\beta_m$-stable with respect to the loss function $\ell$ if the following holds
\begin{equation*}
    \forall S \in \sZ^m,  \forall \bz \in \sZ, \forall i \in [m], \sup_{\bz} \abs{\ell(\sA(S), \bz) - \ell(\sA(S^i), \bz)} \leq \beta_m.
\end{equation*}
\end{mydef}

Given a $\beta_m$-stable learning algorithm, the milestone work~\cite{DBLP:journals/jmlr/BousquetE02} provides a high probability generalization bound that converges when $\beta_m = o(1/\sqrt{m})$. This condition may fail to hold in some modern machine learning settings~\cite {DBLP:conf/nips/XingSC21}, which leads to meaningless guarantees. In recent years, some works~\cite{DBLP:conf/nips/FeldmanV18, DBLP:conf/colt/FeldmanV19, DBLP:conf/colt/BousquetKZ20} improved the classical bound by establishing novel and tighter concentration inequalities. Especially,~\cite{DBLP:conf/colt/BousquetKZ20} proposed a moment bound and obtained a nearly optimal generalization guarantee, which only requires $\beta_m = o(1 / \log m)$ to converge. It is listed in the next theorem.



\begin{theorem}[Corollary 8,~\cite{DBLP:conf/colt/BousquetKZ20}]
\label{thm: classical stability bound}
Assume that $\sA$ is a $\beta_m$-stable learning algorithm and the loss function $\ell$ is bounded by $M$. Given a training set $S$ with $m$ i.i.d. examples sampled from the distribution $\sD$, then for any $\delta \in (0,1)$, with probability at least $1-\delta$, it holds that
\begin{equation}
\abs{\sR_{\sD}(\sA(S)) - \widehat{\sR}_{S}(A(S))} \lesssim \log (m) \beta_m \log \left(\frac{1}{\delta}\right)+ M \sqrt{\frac{1}{m}\log \left(\frac{1}{\delta}\right)}.
\end{equation}
\end{theorem}
We note that all generalization bounds mentioned above require a primary condition: data points are drawn i.i.d. according to the population distribution $\sD$. However, it no longer holds in the setting of GDA. On the one hand, the distribution $\DG$ learned by the generative model is generally not the same as the true distribution $\sD$. On the other hand, the learned $\DG$ is heavily dependent on the sampled dataset $S$. This property brings obstacles to the derivation of the generalization bound for GDA. Furthermore, though there exists some non-i.i.d. stability bounds~\cite{DBLP:conf/nips/MohriR07,DBLP:journals/jmlr/MohriR10,DBLP:conf/nips/ZhangL0W19}, it is still unclear whether these techniques are suitable in the GDA setting.

\section{Main theoretical results}
\label{sec: main results}

In this section, we present our main theoretical results. In Section~\ref{sec: General generalization bound}, we establish a general generalization bound (Theorem~\ref{thm: main generalization bound}) for GDA. Built upon the general learning guarantee, we then particularize the learning setup to the bGMM introduced in Section~\ref{sec: Preliminaries bGMM} and derive a specified generalization bound (Theorem~\ref{thm: bGMM generalization bound}). Finally, we discuss our theoretical implications on GANs in real-world problems (Theorem~\ref{thm: GAN generalization bound}). Notably, to the best of our knowledge, this is the first work to investigate the generalization guarantee of GDA.

\subsection{General generalization bound}
\label{sec: General generalization bound}

To understand GDA, we are interested in studying the generalization error of the hypothesis $\sA(\Saug)$ learned on the dataset $\Saug$ after augmentation. Formally, we need to bound $\vert\sR_{\sD}(\sA(\Saug)) - \widehat{\sR}_{\Saug}(\sA(\Saug))\vert$, which has been defined as \textit{Gen-error} in Section~\ref{sec: Preliminaries-GA}. Recall that $\Daug(S)$ has been defined as the mixed distribution after augmentation, to derive such a bound, we first decomposed \textit{Gen-error} as 
\begin{align*}
\abs{\textit{{Gen-error}}} \leq \underbrace{\abs{\sR_{\sD}(\sA(\Saug)) - \sR_{\Daug(S)}(\sA(\Saug))}}_{\text{Distributions' divergence}} + \underbrace{\abs{ \sR_{\Daug(S)}(\sA(\Saug)) - \widehat{\sR}_{\Saug}(\sA(\Saug))}}_{\text{Generaliztion error w.r.t. mixed distribution}}.
\end{align*}
The first term on the right hand can be bounded by the divergence (e.g., $d_{\mathrm{TV}}, d_{\mathrm{KL}}$) between the mixed distribution $\Daug(S)$ and the true distribution $\sD$. It is heavily dependent on the ability of the chosen generative model. For the second term, we note that classical stability bounds (e.g. Theorem~\ref{thm: classical stability bound}) can not be used directly, because points in $\Saug$ are drawn non-i.i.d.. We mainly use a core property of $\Saug$, that is, $S$ satisfies the i.i.d. assumption, and $S_G$ satisfies the conditional i.i.d. assumption when $S$ is fixed. Inspired by this property, we furthermore decompose this term and utilize sharp moment inequalities~\cite{boucheron2013concentration, DBLP:conf/colt/BousquetKZ20} to obtain an upper bound. Finally, we conclude with the following result.

\begin{theorem}[Generalization bound for GDA, proof in Appendix~\ref{proof: thm: main generalization bound}]
\label{thm: main generalization bound}
Assume that $\sA$ is a $\beta_m$-stable learning algorithm and the loss function $\ell$ is bounded by $M$. Given an set $\Saug$ augmented as described in Section~\ref{sec: Preliminaries-GA}, then for any $\delta \in (0,1)$, with probability at least $1-\delta$, it holds that

\begin{align*}
\abs{\text{Gen-error}} 
&\lesssim \underbrace{\frac{m_G}{m_T}M d_{\mathrm{TV}} \left(\sD , \DG \right)}_{\text{Distributions’ divergence}} + \frac{M(\sqrt{m_S} + \sqrt{m_G}) + m_S\sqrt{m_G}\beta_{m_T} }{m_T} \sqrt{\log \left(\frac{1}{\delta}\right)}\\
&+\frac{\beta_{m_T} \left(m_S \log m_S + m_G \log m_G\right) + m_S \log m_S M \sT(m_S, m_G) }{m_T} \log \left(\frac{1}{\delta}\right),
\end{align*}
where $\sT(m_S, m_G) = \sup_i d_{\mathrm{TV}}\left(\sD_G^{m_G}(S), \sD_G^{m_G}(S^i) \right)$. 

\end{theorem}

\begin{remark}
\textbf{Tightness of the proposed upper bound.} Let $m_G = 0$, we observe that Theorem~\ref{thm: main generalization bound} degenerates to Theorem~\ref{thm: classical stability bound}. Therefore, our stability bound includes the i.i.d. setting as a special case and benefits from the same nearly optimal guarantee shown by~\cite{DBLP:conf/colt/BousquetKZ20}. Further analysis of the tightness of our guarantee when $m_G > 0$ is left to future work.
\end{remark}

\begin{remark}
\textbf{Comparison with the existing non-i.i.d. stability bounds.} Detailed introduction for non-i.i.d. stability bounds is placed in Section~\ref{sec: related work}. We note that previous results~\cite{DBLP:conf/nips/MohriR07,DBLP:journals/jmlr/MohriR10,DBLP:conf/nips/ZhangL0W19} are proposed for the general non-i.i.d. case. Therefore, they may fail to give awesome guarantees in this special case. In Appendix~\ref{sec: Discussion for existing non-i.i.d. stability bounds}, we show that it is hard to derive a better bound than Theorem~\ref{thm: main generalization bound} by using the existing non-i.i.d. stability results directly.
\end{remark}

\begin{remark}
\textbf{Stability of the learned distribution $\DG$.}  $\sT(m_S, m_G)$ in Theorem~\ref{thm: main generalization bound} reflects the stability of the learned distribution with regard to changing one data point in the training set received by the generative model. Our bound suggests that the more stable the model distribution is, the better performance can be achieved by GDA. As far as we know, though uniformly stability of some generative learning algorithms has been studied~\cite{DBLP:conf/icml/FarniaO21}, the new notation $\sT(m_S, m_G)$ emerging in our bound has not been studied yet.
\end{remark}


\begin{remark}
\textbf{Selection of augmentation size.} We first consider the order of the upper bound with respect to $m_S$. Observing Theorem~\ref{thm: main generalization bound}, we find that the distributions' divergence term can not be controlled by increasing $m_G$ while the remaining generalization error w.r.t. mixed distribution will vanish. We note that there exists a trade-off between the fast learning rate and augmentation consumption. When the order of the divergence term is smaller than that of the remains, increasing $m_G$ can induce a faster convergence. Otherwise, increasing $m_G$ can not lead to a faster convergence but a larger consumption. Therefore, an efficient augmentation size $m_{G, \mathrm{order}}^*$ with regard to the order of $m_S$ can be defined as follows:
\begin{equation*}
    m_{G, \mathrm{order}}^* = \inf_{m_G}\left\{ \text{generalization error w.r.t. mixed distribution} \lesssim \text{distributions' divergence} \right\}.
\end{equation*}
Furthermore, without considering the cost, the optimal augmentation number $m_G^*$ can be achieved by minimizing the upper bound directly. Unfortunately, it is difficult to calculate an explicit form of $m_{G, \mathrm{order}}^*$ and $m_G^*$ here due to the ignorance of $\beta_{m_T}$ and $\sT(m_S, m_G)$. We will discuss them more concretely in the specified cases.
\end{remark}

\begin{remark}
\textbf{Sufficient conditions for GDA with (no) faster learning rate.} We still consider the order of the learning guarantee with respect to $m_S$ here. Let $m_G = m_{G, \mathrm{order}}^*$, it can be found that divergence $d_{\mathrm{TV}} \left(\sD, \DG \right)$ plays an important role in deciding whether GDA can enjoy a faster learning rate. Comparing Theorem~\ref{thm: main generalization bound} with Theorem~\ref{thm: classical stability bound} (without augmentation), we can conclude sufficient conditions as follows.

\begin{cor}
Assume that the loss function $\ell$ is bounded by $M$, we have
\begin{itemize}
    \item if $d_{\mathrm{TV}} \left(\sD, \DG \right) = o\left(\max\left( \log(m)\beta_m, 1 / \sqrt{m})\right)\right)$, then GDA enjoys a faster learning rate.
    \item if $d_{\mathrm{TV}} \left(\sD, \DG \right) = \Omega\left(\max\left( \log(m)\beta_m, 1 / \sqrt{m})\right)\right)$, then GDA can not enjoy a faster learning rate.
\end{itemize}
\end{cor}
Notably, as we will present in Section~\ref{sec: theory Results on bGMM} and~\ref{sec: Implications on deep generative models}, though GDA can not enjoy a faster learning rate in the second case, it is possible to improve the generalization guarantee at a constant level when $m_S$ is small, which is important when awful overfitting happens.
\end{remark}

\subsection{Theoretical results on bGMM}
\label{sec: theory Results on bGMM}

The bGMM is a classical but non-trivial setting, which has been widely studied in literature~\cite{DBLP:journals/tit/CastelliC96,DBLP:journals/neco/AkahoK00,wang2022binary}. In this section, we investigated it in the context of GDA. Simulations will be conducted in Section~\ref{sec: bGMM exp} to verify these results.

\subsubsection{Setting of bGMM}
\label{sec: Preliminaries bGMM}

In this part, we introduce the data distribution configuration in the bGMM, as well as the corresponding linear classifier and conditional generative model. Similar setups of distribution and classifier have been adopted by many previous works~\cite{he2022information,DBLP:conf/nips/SchmidtSTTM18,DBLP:conf/nips/AlayracUHFSK19}.

\textbf{Distribution setting.} We consider a binary task where $\sY = \{-1,1\}$. Given a vector $\boldsymbol{\mu} \in \R^d (\Vert \boldsymbol{\mu}\Vert_2 = 1)$ and noise variance $\sigma^2 > 0$, we assume that the distribution satisfies $y \sim \mathrm{uniform}\{-1,1\}$ and $\bx \mid y \sim \mathcal{N}(y\boldsymbol{\mu}, \sigma^2 I_d)$. Besides, similarly to~\cite{DBLP:journals/corr/abs-2212-00362}, we assume that the distribution of $y$ is known, which is satisfied in conditional learning with labels.

\textbf{Simple linear classifier.} We consider a linear classifier parameterized by $\boldsymbol{\theta} \in \R^d$ in the form of prediction $\widehat{y} = \sign(\boldsymbol{\theta}^{\top} \bx)$. Given $m$ samples, $\boldsymbol{\theta}$ is learned by performing ERM with respect to the negative log-likelihood loss function, that is,
$$
l(\boldsymbol{\theta},(\bx, y)) = \frac{1}{2 \sigma^2}(\bx-y \boldsymbol{\theta})^{\top}(\bx-y \boldsymbol{\theta}).
$$
As a result, this learning algorithm will return $\widehat{\boldsymbol{\theta}} = \frac{1}{m} \sum_{i=1}^m y_i\bx_i$, which satisfies $\E[\widehat{\boldsymbol{\theta}}] = \boldsymbol{\mu}$.

\textbf{Conditional generative model.} We consider a simple generative model parameterized by $\boldsymbol{\mu}_y, \sigma^2_k$, where $y \in \{-1,1\}$ and $k \in [d]$. It learns the parameters of Gaussian mixture distribution directly. Given $m$ data points, let $m_y$ be the number of samples in class $y$, it returns
\begin{align*}
\widehat{\boldsymbol{\mu}}_y = \frac{\sum_{y_i = y} \bx_i}{m_y}, \quad \widehat{\sigma}^2_k = \sum_{y} \frac{m_y}{m} \frac{\sum_{y_i = y}(x_{ik} - \widehat{\mu}_{yk})^2}{m_y - 1},
\end{align*}
which are unbiased estimators of $\pm \boldsymbol{\mu}$ and $\sigma^2$, respectively. Based on the learned parameters, we can perform GDA by generating new samples from the distribution $y \sim \mathrm{uniform}\{-1,1\}$, $\bx \mid y \sim \mathcal{N}(\widehat{\boldsymbol{\mu}}_y, \Sigma)$, where $\Sigma = \mathrm{diag}(\sigma^2_1, \dots, \sigma^2_d)$.

\subsubsection{Theoretical results}

In this section, we establish the generalization bound for bGMM based on the general bound proposed in Theorem~\ref{thm: main generalization bound}. To derive such a bound, the main task is to bound terms $M$, $\beta_{m_T}$, $d_{\mathrm{TV}} \left(\sD, \DG \right)$ and $\sT(m_S, m_G)$ in Theorem~\ref{thm: main generalization bound}. For $M$ (Lemma~\ref{lemma: M}) and $\beta_{m_T}$ (Lemma~\ref{lemma: beta}), we mainly use the concentration property of the multivariate Gaussian variable (Lemma~\ref{lemma: gaussian bounded}). In addition, inspired by previous works on na\"ive Bayes~\cite{DBLP:conf/nips/NgJ01}, we bound $d_{\mathrm{TV}} \left(\sD, \DG \right)$ (Lemma~\ref{lemma: kl}) by discussing the distance between the estimated parameters and the true parameters of bGMM. Besides, the concentration property of $\sT(m_S, m_G)$ (Lemma~\ref{lemma: tau}) can be induced by the preceding discussion. Finally, we can obtain the following results.

\begin{theorem}[Generalization bound for bGMM, proof in Appendix \ref{proof: thm bGMM generalization bound}]
\label{thm: bGMM generalization bound}

Consider the setting introduced in Section~\ref{sec: Preliminaries bGMM}. Given a set $S$ with $m_S$ i.i.d. samples from the bGMM distribution $\sD$ and an augmented set $S_G$ with  $m_G$ i.i.d. samples drawn from the learned Gaussian mixture distribution, then with high probability at least $1-\delta$, it holds that

\begin{align}
&\abs{\textit{Gen-error}} \lesssim \text{(\ref{eqn: upper bound bgmm}) in Appendix~\ref{proof: thm bGMM generalization bound}}
\lesssim \begin{cases} 
\frac{\log(m_S)}{\sqrt{m_S}} & \text{ if fix $d$ and $m_G = 0$,}\\
\frac{\log^2(m_S)}{\sqrt{m_S}} & \text{ if fix $d$ and $m_G = \Theta(m_S)$,}\\
\frac{\log(m_S)}{\sqrt{m_S}} & \text{ if fix $d$ and $m_G = m_{G, \mathrm{order}}^*$,}\\
d & \text{ if fix $m_S$.}
\end{cases} \label{eqn: bGMM thm}
\end{align}
\end{theorem}

\begin{remark}
\textbf{Explicit upper bound of generalization error.} (\ref{eqn: upper bound bgmm}) give us an explicit form to predict the generalization error in the bGMM setting. The optimal augmentation size $m_G^*$ can be obtained by minimizing it. In Section~\ref{sec: bGMM exp}, we will see that (\ref{eqn: upper bound bgmm}) predicts the order and trend of true generalization error well, which verifies the correctness of the proposed learning guarantee in the bGMM setting.
\end{remark}

\begin{remark}
\textbf{Negative learning rate of GDA.} Even though we estimate the sufficient statistic of the Gaussian mixture distribution ($\boldsymbol{\mu}$ and $\sigma^2$) directly in this special case, we can not hope to enjoy a better learning rate when $m_G = m_{G, \mathrm{order}}^*$. Things could be worse when we model the distribution in reality (e.g., images, texts), which suggests that when original samples are abundant, further performing GDA can not improve the generalization. Theorem~\ref{thm: GAN generalization bound} supports this viewpoint.
\end{remark}

\begin{remark}
\textbf{Improvement at a constant level matters a lot when overfitting happens.} From (\ref{eqn: bGMM thm}) we know that when $m_S$ is small and $d$ is large, the curse of dimensionality happens, which leads to an awful generalization error. In this case, though GDA can only improve it at a constant level by controlling the generalization error w.r.t. mixed distribution, the effect is obvious due to the large scale of $d$. 
\end{remark}

\subsection{Implications on deep generative models}
\label{sec: Implications on deep generative models}

Nowadays, data augmentation with deep generative models is widely used and received lots of attention. Therefore, benefiting from the recent advances in the generative adversarial network (GAN)~\cite{goodfellow2020generative,DBLP:journals/jmlr/Liang21} and SGD~\cite{DBLP:conf/uai/ZhangZBP0022,DBLP:journals/corr/mingzewang}, we discuss implications of our theory on real problems, which will be verified by the empirical experiments in Section~\ref{sec: cifar10 exp}.

\subsubsection{Learning setup}
\label{sec: Learning setup gan}

We consider the general binary classification task in the deep learning era. In this part, we introduce the setup of data distribution, deep neural classifier, learning algorithm, and deep generative model.

\textbf{Distribution setting.} We assume that input space satisfies $\sX \subseteq [0,1]^d$, and our analysis can be easily extended to any bounded input space. This assumption generally holds in many practical problems, for example, image data satisfies $\sX \subseteq [0,255]^d$. Similarly to bGMM, we let $\sY = \{-1,1\}$ and assume that the distribution of $y$ is known.

\textbf{Deep neural classifier.} We consider a general $L$-layer multi-layer perception (MLP) or convolutional neural network (CNN) $f(\bw, \cdot): \sZ \rightarrow \R$, where $\bw$ denotes its weights and $\bw_l$ denotes the weights in the $l$-th layer. Its abstract architecture is consistent with that in~\cite{DBLP:journals/corr/mingzewang}, and details can be found in Appendix~\ref{sec: deep classifier arch}. In addition, we suppose the deep neural classifier satisfies smoothness and boundedness assumptions, which are adopted by many previous works~\cite{DBLP:conf/icml/HardtRS16, DBLP:conf/uai/ZhangZBP0022, DBLP:journals/corr/mingzewang, DBLP:conf/nips/BartlettFT17}.

\begin{assumption}[Smoothness]
We assume that $f(\bw, \cdot)$ is $\eta$-smooth with respect to $\bw$, that is, $\vert \nabla f(\bw_1, \cdot) - \nabla f(\bw_2, \cdot) \vert \leq \eta \Vert\bw_1-\bw_2\Vert_2$ for any $\bw_1$ and $\bw_2$.
\end{assumption}

\begin{assumption}[Boundedness]
\label{ass: param bounded}
We assume that for all $l \in [L]$, there exists a constant $W_l$, which satisfies $\Vert \bw_l \Vert_2 \leq W_l$.
\end{assumption}

\textbf{Learning algorithm for the deep neural classifier.} The setting of the learning algorithm is conformed to the practice. We assume that the loss function is the binary cross-entropy loss $\ell(f, (\bx, y)) = \log(1 + \exp(-y f(\bw, \bx)))$ and it is optimized by SGD. For the $t$-th step, we set the learning rate as $\frac{c}{\eta t}$ for some positive constant $c$. Besides, we assume that the total iteration number $T = O(m_T)$. These configurations are adopted by past works on the stability of SGD~\cite{DBLP:conf/icml/HardtRS16, DBLP:conf/uai/ZhangZBP0022}.

\textbf{Deep generative model.} We choose GAN as our deep generative model, which is parameterized by MLP. Its abstract architecture is the same as that in Theorem 19 of~\cite{DBLP:journals/jmlr/Liang21}, and details are placed in Appendix~\ref{sec: GAN arch}. Besides, due to the lack of conditional generative model theory, we make a naive approximation here by assuming that each category is learned by a GAN, respectively.

\subsubsection{Theoretical results}

Similarly to the bGMM setting, we establish a generalization bound for the deep learning setup. To reach this goal, we bound terms $M$, $\beta_{m_T}$, and $d_{\mathrm{TV}} \left(\sD, \DG \right)$ based on the recent results on GAN~\cite{DBLP:journals/jmlr/Liang21} and SGD~\cite{DBLP:conf/nips/ZhangL0W19,DBLP:journals/corr/mingzewang}. First, boundedness and Lipschitzness of classifier $f$ can be induced from Assumption~\ref{ass: param bounded} (Lemma~\ref{lemma: Upper bounds for output and gradient}). Second, the boundedness of $f$ directly implies the upper bound for $M$ because the binary cross-entropy loss is 1-Lipschitz with respect to $f$. Third, by combining the Lipschitzness and smoothness of $f$, we can bound $\beta_{m_T}$ for SGD (Lemma~\ref{lemma: SGD stability}). Finally, $d_{\mathrm{TV}} \left(\sD, \DG \right)$ can be bounded by the result in~\cite{DBLP:journals/jmlr/Liang21} (Lemma~\ref{lemma: Learnability of GAN}).

\begin{theorem}[Generalization bound for GAN, proof in Appendix~\ref{sec: proof of GAN generalization bound} ]
\label{thm: GAN generalization bound}
Consider the setup introduced in Section~\ref{sec: Learning setup gan}. Given a set $S$ with $m_S$ i.i.d. samples from any distribution $\sD$ and an augmented set $S_G$ with $m_G$ i.i.d. examples sampled from the distribution $\DG$ learned by GANs, then for any fixed $\delta \in (0,1)$, with probability at least $1-\delta$, it holds that

\begin{align*}
\E \abs{\textit{Gen-error}}
 \lesssim \begin{cases} 
\frac{1}{\sqrt{m_S}} & \text{ if fix $W, L, d$, let $m_G = 0$,}\\
\max\left(\left(\frac{\log(m_S)}{m_S} \right)^{\frac{1}{4}}, { \log m_S  \sT(m_S, m_G) }\right) & \text{ if fix $W, L, d$, let $m_G = \Theta(m_S)$,}\\
\left(\frac{\log(m_S)}{m_S} \right)^{\frac{1}{4}} & \text{ if fix $W, L, d$, let $m_G = m_{G, \mathrm{order}}^*$,}\\
{d} L^2 \left(  \prod_{l=1}^L\left\| W_l \right\|_2 \right)^2 & \text{ if fix $m_S$.} 
\end{cases}
\end{align*}

\end{theorem}

\begin{remark}
\textbf{Slow learning rate with GDA.} Upper bounds in Theorem~\ref{thm: GAN generalization bound} show that when we perform GDA, the order with regard to $m_S$ strictly becomes worse. Therefore, it implies that when $m_S$ is large enough, it is hopeless to boost the performance obviously by augmenting the train set based on GANs. On the contrary, GDA may make the generalization worse.
\end{remark}

\begin{remark}
\textbf{GDA matters a lot when overfitting happens.} From Theorem~\ref{thm: GAN generalization bound}, we know that as the data dimension and model capacity become larger, the deep neural classifier trained with SGD becomes easier to overfit the train set and gain terrible generalization performance. In this case, a constant-level improvement of generalization caused by GDA will be significant.
\end{remark}

\section{Experiments}
\label{sec: exp}

In this section, we conduct experiments to verify the results in Section~\ref{sec: main results}, which are two-folded:

\begin{itemize}
    \item We conduct simulations in the setting of bGMM and validate the results in Theorem~\ref{thm: bGMM generalization bound}.
    \item We empirically study the effect of GDA on the real CIFAR-10 dataset~\cite{cifar}, which supports our theoretical implications on GANs.
\end{itemize}

\subsection{Simulations on bGMM}
\label{sec: bGMM exp}
We let $\boldsymbol{\mu} = (1/\sqrt{d}, \dots, 1/\sqrt{d})^{\top}$ to satisfy $\Vert\boldsymbol{\mu}\Vert_2 = 1$, $\sigma^2 = 0.6^2$, and randomly generate 10,000 samples according to the Gaussian mixture distribution as the test set. We approximate the \textit{Gen-error} by the gap between the training error and the test error. To eliminate randomness, we average over 1,000 random runs and report the mean results. We denote $\gamma = m_G/m_S$ in this section.

First, we investigate the case that data dimension $d$ is fixed. To verify the order is near to $O(1/ \sqrt{m_S})$ ($\log m_S$ can be ignored with respect to $\sqrt{m_S}$), we fix $d = 1$, and change $m_S$ from 20 to 500. For each selected $m_S$, we adjust $\gamma$ from 0 to 50 to generate new samples in different levels. The result is presented in Figure~\ref{fig: sim a}, which shows that the generalization error decreases in a near $O(1/ \sqrt{m_S})$ order. Besides, generalization error without GDA is always (near) optimal, which empirically proves that GDA is ineffective when $m_S$ is large enough.

Second, we conduct simulations in the case that $m_S$ is fixed as a small constant. To verify the order is $O(d)$, we fix $m_S = 10$, and change $d$ from 2 to 100. For each selected $d$, we also adjust $\gamma$ from 0 to 50. The result is displayed in Figure~\ref{fig: sim d}, which shows that the generalization error increases in a $O(d)$ order. In addition, when $d$ is large (e.g., 100) and  the curse of dimensionality happens, generalization error with larger $\gamma$ is better by a big margin, which suggests that though GDA could only enhance it at a constant level, the effect is significant when overfitting occurs.

Third, we design experiments to validate whether the upper bound in Theorem~\ref{thm: bGMM generalization bound} can predict the trend of generalization error well. Similarly to previous theoretical works (e.g.~\cite{DBLP:journals/ml/Ben-DavidBCKPV10}), we find an approximation of~(\ref{eqn: upper bound bgmm}) in Appendix~\ref{proof: thm bGMM generalization bound} as our prediction by replacing $\log(a/\delta)$ with $\log(a)$ if $a \ne 1$ else 1. We plot the ground truths and predictions in the case that $(d, m_S) = (1, 40)$ and $(50,10)$, respectively. Results in Figure~\ref{figures: simulations} show that our bound predicts the trend of generalization error well. Therefore, an approximation of the optimal augmentation size $m_G^*$ can be found by minimizing (\ref{eqn: upper bound bgmm}).

\begin{figure}[t]
\centering

\subfloat[$d = 1$, truth]{
\includegraphics[width=0.3\columnwidth,height=0.2\columnwidth]{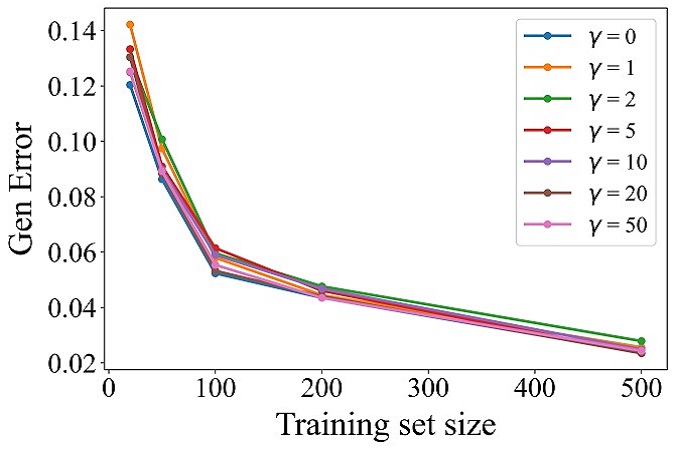}
\label{fig: sim a}
}%
\subfloat[$(d, m_S) = (1, 40)$, truth]{
\includegraphics[width=0.3\columnwidth,height=0.2\columnwidth]{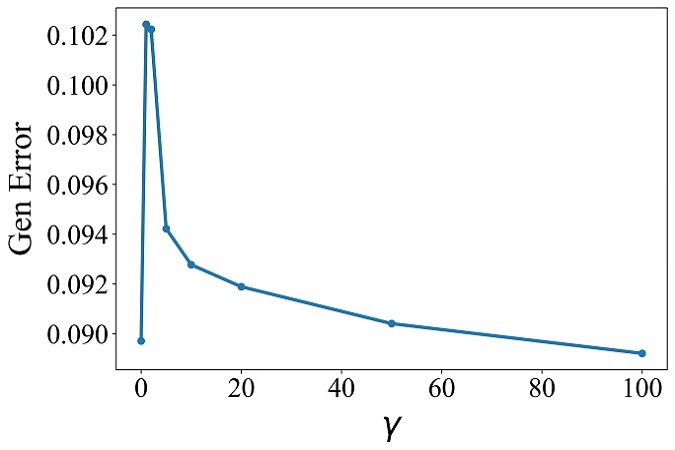}
\label{fig: sim b}
}%
\subfloat[$(d, m_S) = (1, 40)$, prediction]{
\includegraphics[width=0.3\columnwidth,height=0.2\columnwidth]{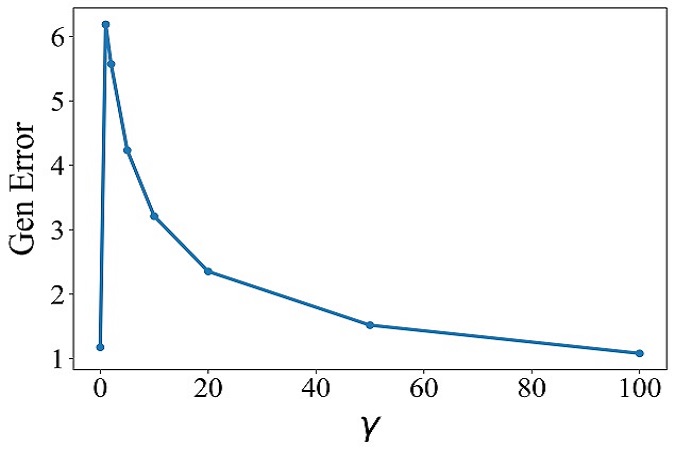}
\label{fig: sim c}
}%

\vskip 0.5ex

\subfloat[$m_S = 10$, truth]{
\includegraphics[width=0.3\columnwidth,height=0.2\columnwidth]{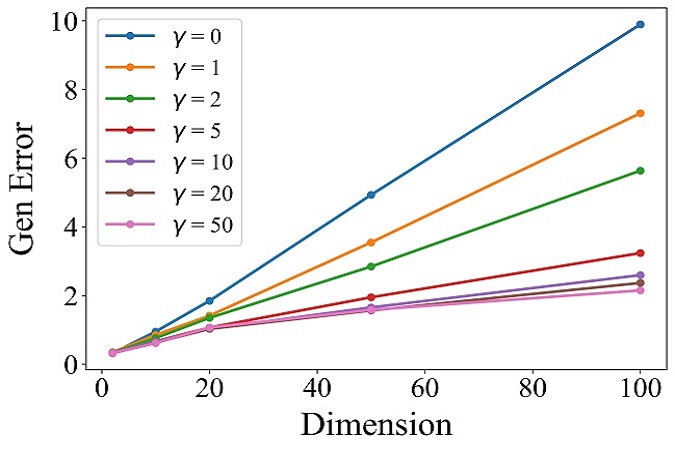}
\label{fig: sim d}
}%
\subfloat[$(d, m_S) = (50, 10)$, truth]{
\includegraphics[width=0.3\columnwidth,height=0.2\columnwidth]{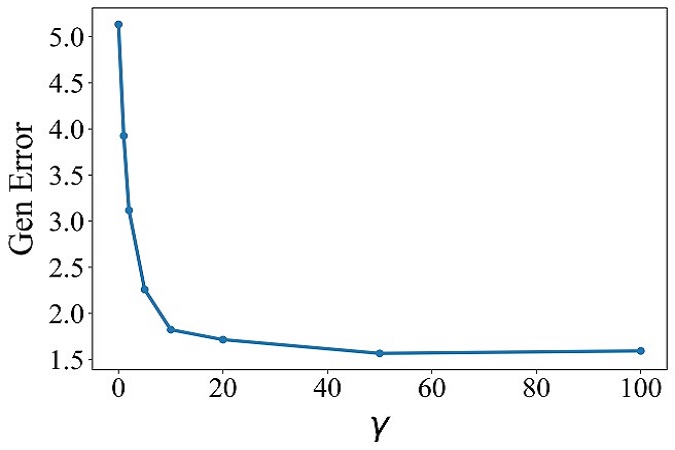}
\label{fig: sim e}
}%
\subfloat[$(d, m_S) = (50, 10)$, prediction]{
\includegraphics[width=0.3\columnwidth,height=0.2\columnwidth]{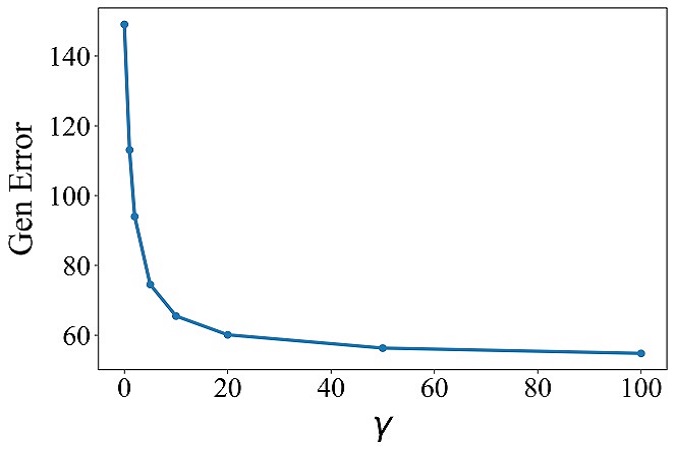}
\label{fig: sim f}
}%

\vskip 1em

\caption{Simulations results on the bGMM setting.}
\vskip -1.2em

\label{figures: simulations}
\end{figure}

\subsection{Empirical results on CIFAR-10}
\label{sec: cifar10 exp}

In this part, we conduct experiments on the real CIFAR-10 dataset with ResNets~\cite{DBLP:conf/cvpr/HeZRS16} and various deep generative models, including conditional DCGAN (cDCGAN)~\cite{DBLP:journals/corr/RadfordMC15}, StyleGAN2-ADA~\cite{DBLP:conf/nips/KarrasAHLLA20} and elucidating diffusion
model (EDM)~\cite{DBLP:journals/corr/EDM}. Details of experiments can be found in Appendix~\ref{sec: Additional experimental details and results}.

To validate our theoretical implications in Section~\ref{sec: Implications on deep generative models}, we are supposed to discuss two cases, where one $m_S$ is small and the other $m_S$ is large. The two cases can be approximated by whether performing another data augmentation. We additionally use the standard data augmentation in~\cite{DBLP:conf/cvpr/HeZRS16} to approximate the case with large $m_S$. Then, for each selected ResNet and generative model, we set $m_G$ from 0 to 1M and record the accuracy of the trained classifier on the CIFAR-10 test set. Results are presented in Table~\ref{tab: deep result} of Appendix~\ref{sec: Additional experimental details and results}. We interpret them as the following.

\textbf{GANs improve the test performance of classifiers when overfitting occurs.} When standard augmentation is not used, ResNets trained on the train set consistently suffer from overfitting. However, this can be relieved by data augmentation based on GANs, though cDCGAN can not generate high-quality images. This phenomenon supports the implications from Theorem~\ref{thm: GAN generalization bound}.

\textbf{We can not have an obvious improvement by using GANs when $m_S$ is approximately large.} When standard augmentation is used, deep neural classifiers trained on the CIFAR-10 dataset achieve non-trivial performance. In this case, GDA with cDCGAN always damages the generalization ability. Though we use StyleGAN2-ADA, which achieves state-of-the-art conditional image generation performance on the CIFAR-10 dataset, we can not boost the performance of classifiers obviously, and even consistently obtain worse test accuracy when $m_G$ is 500k or 1M.

\textbf{Diffusion probabilistic models are promising for GDA.} As diffusion models show their excellent ability on image generation, a natural question emerges: \textit{are diffusion models more suitable for GDA?} We choose the EDM that achieves state-of-the-art FID scores as the generator. Table~\ref{tab: deep result} in Appendix~\ref{sec: Additional results} shows that EDM improves the test accuracy obviously, even though the standard augmentation has been utilized. This suggests that diffusion models enjoy $d_{\mathrm{TV}} \left(\sD, \DG \right)$ with a faster convergence rate than GANs, and shows the promise of diffusion models in  GDA.

\section{Related work}
\label{sec: related work}

\textbf{Data augmentation practice and theory.} Data augmentation~\cite{DBLP:journals/jbd/ShortenK19,DBLP:journals/jbd/ShortenKF21a} is a universal method to improve the generalization ability of deep neural networks in the case of insufficient training data. Classical data augmentation methods include geometric transformations~\cite{DBLP:conf/cvpr/HeZRS16}, color space transformations~\cite{jurio2010comparison}, kernel filters~\cite{kang2017patchshuffle}, mixing images~\cite{inoue2018data}, random erasing~\cite{zhong2020random}, feature space augmentation~\cite{devries2017dataset}, etc. There are also many theoretical works studying the effect of classical data augmentation methods from different perspectives~\cite{DBLP:conf/icml/DaoGRSSR19,DBLP:conf/icml/WuZVR20,DBLP:conf/nips/HaninS21,DBLP:conf/icml/ShenBG22,DBLP:journals/corr/abs-2303-08433}.

With the advance of deep generative models, GDA becomes a novel and promising data augmentation technique. For example,~\cite{DBLP:journals/corr/abs-2304-08466} shows that augmenting the ImageNet training set~\cite{imagenet} with samples from the conditional diffusion models significantly improves the classification accuracy. However, little work has investigated the theory of GDA. Both empirical success and theoretical opening encourage us to study the role of GDA.

\textbf{Algorithmic stability theory.} Classical results~\cite{DBLP:journals/jmlr/BousquetE02,DBLP:conf/colt/BousquetKZ20} introduced detailedly in Section~\ref{sec: Preliminaries} has various extensions. Prominent work~\cite{DBLP:conf/icml/HardtRS16} focuses on the uniform stability of SGD and derive generalization bounds for it.~\cite{DBLP:conf/uai/ZhangZBP0022} improves the results in~\cite{DBLP:conf/icml/HardtRS16} and obtains tight guarantees for the stability of SGD, which is used in Theorem~\ref{thm: GAN generalization bound}.

Establishing stability bounds under non-i.i.d. settings has also received a surge of interest in recent years. A major line models the dependencies by mixing models~\cite{rosenblatt1956central,volkonskii1959some} and derives stability bounds with mixing coefficients~\cite{DBLP:conf/nips/MohriR07,DBLP:journals/jmlr/MohriR10,DBLP:journals/ijon/HeZC16}. However, it is usually difficult to estimate the mixing coefficients quantitatively. To avoid this problem, another line qualitatively models the dependencies by graphs. Recently,~\cite{DBLP:conf/nips/ZhangL0W19} derive a general stability bound for dependent settings characterized by forest complexity of the dependency graph. However, it is hard to use these techniques to derive a better bound than Theorem~\ref{thm: main generalization bound} for GDA, which is discussed detailedly in Appendix~\ref{sec: Discussion for existing non-i.i.d. stability bounds}.

\textbf{Convergence of deep generative models.} In addition to the bound for $d_{\mathrm{TV}} \left(\sD, \DG \right)$ with respect to GANs~\cite{DBLP:journals/jmlr/Liang21} we used in Theorem~\ref{thm: GAN generalization bound}, there are attempts to derive such a bound for diffusion models~\cite{DBLP:journals/corr/abs-2209-11215,DBLP:conf/alt/Lee0T23,DBLP:journals/corr/abs-2208-14699,DBLP:journals/corr/abs-2208-05314}. Informally, they mainly assume that estimation error of score function is bounded, then with an appropriate choice of step size and iteration number, diffusion models output a distribution which is close to the true distribution. However, it is still unclear how to derive learning guarantees with respect to the train set size $m_S$ directly. Once such learning guarantees are established, we can directly analyze the effect of GDA with diffusion models by Theorem~\ref{thm: main generalization bound}.

\section{Conclusion}
\label{sec: conclusion}
In this paper, 
we attempt to understand modern GDA techniques. To realize this goal, we first establish a general algorithmic stability bound in this non-i.i.d. setting. It suggests that GDA enjoys a faster learning rate when the divergence term $d_{\mathrm{TV}} \left(\sD, \DG \right) = o(\max\left( \log(m)\beta_m, 1 / \sqrt{m})\right)$. Second, We specify the learning guarantee to the bGMM and GANs settings. Theoretical results show that, in both cases, though GDA can not enjoy a faster learning rate, it is effective when terrible overfitting happens, which suggests its promise in learning with limited data. Finally, experimental results support our theoretical conclusions and further show the promise of diffusion models in GDA.

\textbf{Broader impacts and limitations.} This is mainly theoretical work to help people understand GDA, and we do not see a direct negative social impact of our theory. One limitation is that results do not enjoy tightness guarantees. The derivation of lower bounds can be left to future work.

\bibliographystyle{unsrt}  
\bibliography{ref}

\begin{thebibliography}{10}

\bibitem{DBLP:journals/corr/vae}
Diederik~P. Kingma and Max Welling.
\newblock Auto-encoding variational bayes.
\newblock In {\em {ICLR}}, 2014.

\bibitem{goodfellow2020generative}
Ian Goodfellow, Jean Pouget-Abadie, Mehdi Mirza, Bing Xu, David Warde-Farley,
  Sherjil Ozair, Aaron Courville, and Yoshua Bengio.
\newblock Generative adversarial networks.
\newblock {\em Communications of the ACM}, 63(11):139--144, 2020.

\bibitem{DBLP:conf/nips/HoJA20}
Jonathan Ho, Ajay Jain, and Pieter Abbeel.
\newblock Denoising diffusion probabilistic models.
\newblock In {\em {NeurIPS}}, 2020.

\bibitem{DBLP:conf/iclr/0011SKKEP21}
Yang Song, Jascha Sohl{-}Dickstein, Diederik~P. Kingma, Abhishek Kumar, Stefano
  Ermon, and Ben Poole.
\newblock Score-based generative modeling through stochastic differential
  equations.
\newblock In {\em {ICLR}}, 2021.

\bibitem{ramesh2021zero}
Aditya Ramesh, Mikhail Pavlov, Gabriel Goh, Scott Gray, Chelsea Voss, Alec
  Radford, Mark Chen, and Ilya Sutskever.
\newblock Zero-shot text-to-image generation.
\newblock In {\em {ICML}}, pages 8821--8831, 2021.

\bibitem{rombach2022high}
Robin Rombach, Andreas Blattmann, Dominik Lorenz, Patrick Esser, and Bj{\"o}rn
  Ommer.
\newblock High-resolution image synthesis with latent diffusion models.
\newblock In {\em {CVPR}}, pages 10684--10695, 2022.

\bibitem{brown2020language}
Tom Brown, Benjamin Mann, Nick Ryder, Melanie Subbiah, Jared~D Kaplan, Prafulla
  Dhariwal, Arvind Neelakantan, Pranav Shyam, Girish Sastry, Amanda Askell,
  et~al.
\newblock Language models are few-shot learners.
\newblock {\em NeurIPS}, 33:1877--1901, 2020.

\bibitem{raffel2020exploring}
Colin Raffel, Noam Shazeer, Adam Roberts, Katherine Lee, Sharan Narang, Michael
  Matena, Yanqi Zhou, Wei Li, and Peter~J Liu.
\newblock Exploring the limits of transfer learning with a unified text-to-text
  transformer.
\newblock {\em Journal of Machine Learning Research}, 21(1):5485--5551, 2020.

\bibitem{leiter2023chatgpt}
Christoph Leiter, Ran Zhang, Yanran Chen, Jonas Belouadi, Daniil Larionov,
  Vivian Fresen, and Steffen Eger.
\newblock Chatgpt: {A} meta-analysis after 2.5 months.
\newblock {\em CoRR}, abs/2302.13795, 2023.

\bibitem{wang2022image}
Wenhui Wang, Hangbo Bao, Li~Dong, Johan Bjorck, Zhiliang Peng, Qiang Liu, Kriti
  Aggarwal, Owais~Khan Mohammed, Saksham Singhal, Subhojit Som, and Furu Wei.
\newblock Image as a foreign language: Beit pretraining for all vision and
  vision-language tasks.
\newblock {\em CoRR}, abs/2208.10442, 2022.

\bibitem{bao2023one}
Fan Bao, Shen Nie, Kaiwen Xue, Chongxuan Li, Shi Pu, Yaole Wang, Gang Yue, Yue
  Cao, Hang Su, and Jun Zhu.
\newblock One transformer fits all distributions in multi-modal diffusion at
  scale.
\newblock {\em CoRR}, abs/2303.06555, 2023.

\bibitem{openai2023gpt}
OpenAI.
\newblock {GPT-4} technical report.
\newblock {\em CoRR}, abs/2303.08774, 2023.

\bibitem{DBLP:journals/corr/abs-2304-08466}
Shekoofeh Azizi, Simon Kornblith, Chitwan Saharia, Mohammad Norouzi, and
  David~J. Fleet.
\newblock Synthetic data from diffusion models improves imagenet
  classification.
\newblock {\em CoRR}, abs/2304.08466, 2023.

\bibitem{DBLP:conf/icassp/BesnierJBCP20}
Victor Besnier, Himalaya Jain, Andrei Bursuc, Matthieu Cord, and Patrick
  P{\'{e}}rez.
\newblock This dataset does not exist: Training models from generated images.
\newblock In {\em {ICASSP}}, pages 1--5, 2020.

\bibitem{DBLP:conf/nips/KingmaMRW14}
Diederik~P. Kingma, Shakir Mohamed, Danilo~Jimenez Rezende, and Max Welling.
\newblock Semi-supervised learning with deep generative models.
\newblock In {\em {NIPS}}, pages 3581--3589, 2014.

\bibitem{DBLP:conf/nips/LiXZZ17}
Chongxuan Li, Taufik Xu, Jun Zhu, and Bo~Zhang.
\newblock Triple generative adversarial nets.
\newblock In {\em {NIPS}}, pages 4088--4098, 2017.

\bibitem{DBLP:journals/corr/abs-2302-10586}
Zebin You, Yong Zhong, Fan Bao, Jiacheng Sun, Chongxuan Li, and Jun Zhu.
\newblock Diffusion models and semi-supervised learners benefit mutually with
  few labels.
\newblock {\em CoRR}, abs/2302.10586, 2023.

\bibitem{DBLP:journals/corr/abs-2302-07944}
Brandon Trabucco, Kyle Doherty, Max Gurinas, and Ruslan Salakhutdinov.
\newblock Effective data augmentation with diffusion models.
\newblock {\em CoRR}, abs/2302.07944, 2023.

\bibitem{DBLP:journals/corr/abs-2210-07574}
Ruifei He, Shuyang Sun, Xin Yu, Chuhui Xue, Wenqing Zhang, Philip H.~S. Torr,
  Song Bai, and Xiaojuan Qi.
\newblock Is synthetic data from generative models ready for image recognition?
\newblock {\em CoRR}, abs/2210.07574, 2022.

\bibitem{DBLP:journals/corr/abs-2103-01946}
Sylvestre{-}Alvise Rebuffi, Sven Gowal, Dan~A. Calian, Florian Stimberg, Olivia
  Wiles, and Timothy~A. Mann.
\newblock Fixing data augmentation to improve adversarial robustness.
\newblock {\em CoRR}, abs/2103.01946, 2021.

\bibitem{DBLP:journals/corr/abs-2302-04638}
Zekai Wang, Tianyu Pang, Chao Du, Min Lin, Weiwei Liu, and Shuicheng Yan.
\newblock Better diffusion models further improve adversarial training.
\newblock {\em CoRR}, abs/2302.04638, 2023.

\bibitem{DBLP:journals/jmlr/BousquetE02}
Olivier Bousquet and Andr{\'{e}} Elisseeff.
\newblock Stability and generalization.
\newblock {\em Journal of Machine Learning Research}, 2:499--526, 2002.

\bibitem{DBLP:conf/colt/BousquetKZ20}
Olivier Bousquet, Yegor Klochkov, and Nikita Zhivotovskiy.
\newblock Sharper bounds for uniformly stable algorithms.
\newblock In {\em {COLT}}, volume 125, pages 610--626, 2020.

\bibitem{DBLP:conf/nips/MohriR07}
Mehryar Mohri and Afshin Rostamizadeh.
\newblock Stability bounds for non-i.i.d. processes.
\newblock In John~C. Platt, Daphne Koller, Yoram Singer, and Sam~T. Roweis,
  editors, {\em {NIPS}}, pages 1025--1032, 2007.

\bibitem{DBLP:journals/jmlr/MohriR10}
Mehryar Mohri and Afshin Rostamizadeh.
\newblock Stability bounds for stationary phi-mixing and beta-mixing processes.
\newblock {\em Journal of Machine Learning Research}, 11:789--814, 2010.

\bibitem{DBLP:conf/nips/ZhangL0W19}
Rui~Ray Zhang, Xingwu Liu, Yuyi Wang, and Liwei Wang.
\newblock Mcdiarmid-type inequalities for graph-dependent variables and
  stability bounds.
\newblock In {\em {NeurIPS}}, pages 10889--10899, 2019.

\bibitem{DBLP:journals/corr/EDM}
Tero Karras, Miika Aittala, Timo Aila, and Samuli Laine.
\newblock Elucidating the design space of diffusion-based generative models.
\newblock {\em CoRR}, abs/2206.00364, 2022.

\bibitem{DBLP:journals/jmlr/Shalev-ShwartzSSS10}
Shai Shalev{-}Shwartz, Ohad Shamir, Nathan Srebro, and Karthik Sridharan.
\newblock Learnability, stability and uniform convergence.
\newblock {\em Journal of Machine Learning Research}, 11:2635--2670, 2010.

\bibitem{DBLP:conf/icml/KuzborskijL18}
Ilja Kuzborskij and Christoph~H. Lampert.
\newblock Data-dependent stability of stochastic gradient descent.
\newblock In {\em {ICML}}, volume~80, pages 2820--2829, 2018.

\bibitem{DBLP:conf/icml/LiuLNT17}
Tongliang Liu, G{\'{a}}bor Lugosi, Gergely Neu, and Dacheng Tao.
\newblock Algorithmic stability and hypothesis complexity.
\newblock In {\em {ICML}}, volume~70, pages 2159--2167, 2017.

\bibitem{DBLP:conf/icml/HardtRS16}
Moritz Hardt, Ben Recht, and Yoram Singer.
\newblock Train faster, generalize better: Stability of stochastic gradient
  descent.
\newblock In {\em {ICML}}, volume~48, pages 1225--1234, 2016.

\bibitem{DBLP:conf/uai/ZhangZBP0022}
Yikai Zhang, Wenjia Zhang, Sammy Bald, Vamsi Pingali, Chao Chen, and Mayank
  Goswami.
\newblock Stability of {SGD:} tightness analysis and improved bounds.
\newblock In {\em {UAI}}, volume 180, pages 2364--2373, 2022.

\bibitem{DBLP:conf/nips/XingSC21}
Yue Xing, Qifan Song, and Guang Cheng.
\newblock On the algorithmic stability of adversarial training.
\newblock In {\em {NeurIPS}}, pages 26523--26535, 2021.

\bibitem{DBLP:conf/nips/FeldmanV18}
Vitaly Feldman and Jan Vondr{\'{a}}k.
\newblock Generalization bounds for uniformly stable algorithms.
\newblock In {\em {NeurIPS}}, pages 9770--9780, 2018.

\bibitem{DBLP:conf/colt/FeldmanV19}
Vitaly Feldman and Jan Vondr{\'{a}}k.
\newblock High probability generalization bounds for uniformly stable
  algorithms with nearly optimal rate.
\newblock In {\em {COLT}}, volume~99, pages 1270--1279, 2019.

\bibitem{boucheron2013concentration}
St{\'e}phane Boucheron, G{\'a}bor Lugosi, and Pascal Massart.
\newblock {\em Concentration inequalities: A nonasymptotic theory of
  independence}.
\newblock Oxford university press, 2013.

\bibitem{DBLP:conf/icml/FarniaO21}
Farzan Farnia and Asuman~E. Ozdaglar.
\newblock Train simultaneously, generalize better: Stability of gradient-based
  minimax learners.
\newblock In {\em {ICML}}, volume 139, pages 3174--3185, 2021.

\bibitem{DBLP:journals/tit/CastelliC96}
Vittorio Castelli and Thomas~M. Cover.
\newblock The relative value of labeled and unlabeled samples in pattern
  recognition with an unknown mixing parameter.
\newblock {\em {IEEE} Trans. Inf. Theory}, 42(6):2102--2117, 1996.

\bibitem{DBLP:journals/neco/AkahoK00}
Shotaro Akaho and Hilbert~J. Kappen.
\newblock Nonmonotonic generalization bias of gaussian mixture models.
\newblock {\em Neural Comput.}, 12(6):1411--1427, 2000.

\bibitem{wang2022binary}
Ke~Wang and Christos Thrampoulidis.
\newblock Binary classification of gaussian mixtures: Abundance of support
  vectors, benign overfitting, and regularization.
\newblock {\em SIAM Journal on Mathematics of Data Science}, 4(1):260--284,
  2022.

\bibitem{he2022information}
Haiyun He, Hanshu Yan, and Vincent~YF Tan.
\newblock Information-theoretic characterization of the generalization error
  for iterative semi-supervised learning.
\newblock {\em Journal of Machine Learning Research}, 23:1--52, 2022.

\bibitem{DBLP:conf/nips/SchmidtSTTM18}
Ludwig Schmidt, Shibani Santurkar, Dimitris Tsipras, Kunal Talwar, and
  Aleksander Madry.
\newblock Adversarially robust generalization requires more data.
\newblock In {\em {NeurIPS}}, pages 5019--5031, 2018.

\bibitem{DBLP:conf/nips/AlayracUHFSK19}
Jean{-}Baptiste Alayrac, Jonathan Uesato, Po{-}Sen Huang, Alhussein Fawzi,
  Robert Stanforth, and Pushmeet Kohli.
\newblock Are labels required for improving adversarial robustness?
\newblock In {\em {NeurIPS}}, pages 12192--12202, 2019.

\bibitem{DBLP:journals/corr/abs-2212-00362}
Fan Bao, Chongxuan Li, Jiacheng Sun, and Jun Zhu.
\newblock Why are conditional generative models better than unconditional ones?
\newblock {\em CoRR}, abs/2212.00362, 2022.

\bibitem{DBLP:conf/nips/NgJ01}
Andrew~Y. Ng and Michael~I. Jordan.
\newblock On discriminative vs. generative classifiers: {A} comparison of
  logistic regression and naive bayes.
\newblock In {\em {NIPS}}, pages 841--848, 2001.

\bibitem{DBLP:journals/jmlr/Liang21}
Tengyuan Liang.
\newblock How well generative adversarial networks learn distributions.
\newblock {\em Journal of Machine Learning Research}, 22:228:1--228:41, 2021.

\bibitem{DBLP:journals/corr/mingzewang}
Mingze Wang and Chao Ma.
\newblock Generalization error bounds for deep neural networks trained by
  {SGD}.
\newblock {\em CoRR}, abs/2206.03299, 2022.

\bibitem{DBLP:conf/nips/BartlettFT17}
Peter~L. Bartlett, Dylan~J. Foster, and Matus Telgarsky.
\newblock Spectrally-normalized margin bounds for neural networks.
\newblock In {\em {NIPS}}, pages 6240--6249, 2017.

\bibitem{cifar}
Alex Krizhevsky, Geoffrey Hinton, et~al.
\newblock Learning multiple layers of features from tiny images.
\newblock Technical report, Canadian Institute for Advanced Research, Toronto,
  ON, Canada, 2009.

\bibitem{DBLP:journals/ml/Ben-DavidBCKPV10}
Shai Ben{-}David, John Blitzer, Koby Crammer, Alex Kulesza, Fernando Pereira,
  and Jennifer~Wortman Vaughan.
\newblock A theory of learning from different domains.
\newblock {\em Machine Learning}, 79(1-2):151--175, 2010.

\bibitem{DBLP:conf/cvpr/HeZRS16}
Kaiming He, Xiangyu Zhang, Shaoqing Ren, and Jian Sun.
\newblock Deep residual learning for image recognition.
\newblock In {\em {CVPR}}, pages 770--778, 2016.

\bibitem{DBLP:journals/corr/RadfordMC15}
Alec Radford, Luke Metz, and Soumith Chintala.
\newblock Unsupervised representation learning with deep convolutional
  generative adversarial networks.
\newblock In {\em {ICLR}}, 2016.

\bibitem{DBLP:conf/nips/KarrasAHLLA20}
Tero Karras, Miika Aittala, Janne Hellsten, Samuli Laine, Jaakko Lehtinen, and
  Timo Aila.
\newblock Training generative adversarial networks with limited data.
\newblock In {\em { NeurIPS }}, 2020.

\bibitem{DBLP:journals/jbd/ShortenK19}
Connor Shorten and Taghi~M. Khoshgoftaar.
\newblock A survey on image data augmentation for deep learning.
\newblock {\em Journal of Big Data}, 6:60, 2019.

\bibitem{DBLP:journals/jbd/ShortenKF21a}
Connor Shorten, Taghi~M. Khoshgoftaar, and Borko Furht.
\newblock Text data augmentation for deep learning.
\newblock {\em Journal of Big Data}, 8(1):101, 2021.

\bibitem{jurio2010comparison}
Aranzazu Jurio, Miguel Pagola, Mikel Galar, Carlos Lopez-Molina, and Daniel
  Paternain.
\newblock A comparison study of different color spaces in clustering based
  image segmentation.
\newblock In {\em Information Processing and Management of Uncertainty in
  Knowledge-Based Systems.}, pages 532--541. Springer, 2010.

\bibitem{kang2017patchshuffle}
Guoliang Kang, Xuanyi Dong, Liang Zheng, and Yi~Yang.
\newblock Patchshuffle regularization.
\newblock {\em CoRR}, abs/1707.07103, 2017.

\bibitem{inoue2018data}
Hiroshi Inoue.
\newblock Data augmentation by pairing samples for images classification.
\newblock {\em CoRR}, abs/1801.02929, 2018.

\bibitem{zhong2020random}
Zhun Zhong, Liang Zheng, Guoliang Kang, Shaozi Li, and Yi~Yang.
\newblock Random erasing data augmentation.
\newblock In {\em {AAAI}}, volume~34, pages 13001--13008, 2020.

\bibitem{devries2017dataset}
Terrance DeVries and Graham~W. Taylor.
\newblock Dataset augmentation in feature space.
\newblock In {\em {ICLR} Workshop Track Proceedings}, 2017.

\bibitem{DBLP:conf/icml/DaoGRSSR19}
Tri Dao, Albert Gu, Alexander Ratner, Virginia Smith, Chris~De Sa, and
  Christopher R{\'{e}}.
\newblock A kernel theory of modern data augmentation.
\newblock In {\em {ICML}}, volume~97 of {\em Proceedings of Machine Learning
  Research}, pages 1528--1537, 2019.

\bibitem{DBLP:conf/icml/WuZVR20}
Sen Wu, Hongyang~R. Zhang, Gregory Valiant, and Christopher R{\'{e}}.
\newblock On the generalization effects of linear transformations in data
  augmentation.
\newblock In {\em {ICML}}, volume 119, pages 10410--10420, 2020.

\bibitem{DBLP:conf/nips/HaninS21}
Boris Hanin and Yi~Sun.
\newblock How data augmentation affects optimization for linear regression.
\newblock In {\em {NeurIPS}}, pages 8095--8105, 2021.

\bibitem{DBLP:conf/icml/ShenBG22}
Ruoqi Shen, S{\'{e}}bastien Bubeck, and Suriya Gunasekar.
\newblock Data augmentation as feature manipulation.
\newblock In {\em {ICML}}, volume 162, pages 19773--19808, 2022.

\bibitem{DBLP:journals/corr/abs-2303-08433}
Difan Zou, Yuan Cao, Yuanzhi Li, and Quanquan Gu.
\newblock The benefits of mixup for feature learning.
\newblock {\em CoRR}, abs/2303.08433, 2023.

\bibitem{imagenet}
Jia Deng, Wei Dong, Richard Socher, Li{-}Jia Li, Kai Li, and Li~Fei{-}Fei.
\newblock Imagenet: {A} large-scale hierarchical image database.
\newblock In {\em {CVPR}}, pages 248--255, 2009.

\bibitem{rosenblatt1956central}
Murray Rosenblatt.
\newblock A central limit theorem and a strong mixing condition.
\newblock {\em Proceedings of the national Academy of Sciences}, 42(1):43--47,
  1956.

\bibitem{volkonskii1959some}
VA~Volkonskii and Yu~A Rozanov.
\newblock Some limit theorems for random functions. i.
\newblock {\em Theory of Probability \& Its Applications}, 4(2):178--197, 1959.

\bibitem{DBLP:journals/ijon/HeZC16}
Fangchao He, Ling Zuo, and Hong Chen.
\newblock Stability analysis for ranking with stationary
  \emph{{\(\varphi\)}}-mixing samples.
\newblock {\em Neurocomputing}, 171:1556--1562, 2016.

\bibitem{DBLP:journals/corr/abs-2209-11215}
Sitan Chen, Sinho Chewi, Jerry Li, Yuanzhi Li, Adil Salim, and Anru~R. Zhang.
\newblock Sampling is as easy as learning the score: theory for diffusion
  models with minimal data assumptions.
\newblock {\em CoRR}, abs/2209.11215, 2022.

\bibitem{DBLP:conf/alt/Lee0T23}
Holden Lee, Jianfeng Lu, and Yixin Tan.
\newblock Convergence of score-based generative modeling for general data
  distributions.
\newblock In {\em {ALT}}, volume 201, pages 946--985, 2023.

\bibitem{DBLP:journals/corr/abs-2208-14699}
Xingchao Liu, Lemeng Wu, Mao Ye, and Qiang Liu.
\newblock Let us build bridges: Understanding and extending diffusion
  generative models.
\newblock {\em CoRR}, abs/2208.14699, 2022.

\bibitem{DBLP:journals/corr/abs-2208-05314}
Valentin~De Bortoli.
\newblock Convergence of denoising diffusion models under the manifold
  hypothesis.
\newblock {\em CoRR}, abs/2208.05314, 2022.

\bibitem{DBLP:journals/corr/XuWCL15}
Bing Xu, Naiyan Wang, Tianqi Chen, and Mu~Li.
\newblock Empirical evaluation of rectified activations in convolutional
  network.
\newblock {\em CoRR}, abs/1505.00853, 2015.

\bibitem{wainwright2019high}
Martin~J Wainwright.
\newblock {\em High-dimensional statistics: A non-asymptotic viewpoint},
  volume~48.
\newblock Cambridge university press, 2019.

\bibitem{DBLP:journals/tit/SasonV16}
Igal Sason and Sergio Verd{\'{u}}.
\newblock f-divergence inequalities.
\newblock {\em {IEEE} Trans. Inf. Theory}, 62(11):5973--6006, 2016.

\bibitem{DBLP:journals/corr/chenyu}
Chenyu Zheng, Guoqiang Wu, Fan Bao, Yue Cao, Chongxuan Li, and Jun Zhu.
\newblock Revisiting discriminative vs. generative classifiers: Theory and
  implications.
\newblock {\em CoRR}, abs/2302.02334, 2023.

\bibitem{pytorch}
Adam Paszke, Sam Gross, Francisco Massa, Adam Lerer, James Bradbury, Gregory
  Chanan, Trevor Killeen, Zeming Lin, Natalia Gimelshein, Luca Antiga, Alban
  Desmaison, Andreas K{\"{o}}pf, Edward~Z. Yang, Zachary DeVito, Martin Raison,
  Alykhan Tejani, Sasank Chilamkurthy, Benoit Steiner, Lu~Fang, Junjie Bai, and
  Soumith Chintala.
\newblock Pytorch: An imperative style, high-performance deep learning library.
\newblock In {\em NeurIPS}, pages 8024--8035, 2019.

\end{thebibliography}

\newpage


\begin{appendices}

\renewcommand{\contentsname}{Contents of Appendix}
\tableofcontents

\addtocontents{toc}{\protect\setcounter{tocdepth}{3}} 

\newpage

\section{Architectures of deep neural networks in Section~\ref{sec: Learning setup gan}}
\label{1}

\subsection{Architecture of deep neural classifier in Section~\ref{sec: Implications on deep generative models}}
\label{sec: deep classifier arch}

We consider a general class of neural networks as what is introduced in~\cite{DBLP:journals/corr/mingzewang}, which includes widely used MLPs and CNNs. We define a deep neural network with $L_C$ convolutional layers followed by $L - L_C - 1$ fully-connected layers as follows:
$$
\begin{aligned}
f(\mathbf{x} ; \bw) & =\sum_{k=1}^m a_k z_{(L-1), k} \\
\mathbf{z}_{l} & =\sigma\left(\mathbf{A}_l^{\top} \mathbf{z}_{(l-1)}\right), l \in[L-1]-\left[L_C\right], \\
\mathbf{z}_{l} & =\operatorname{pool}\left(\mathbf{y}_{l}\right), l \in\left[L_C\right], \\
\mathbf{y}_l & =\sigma\left(\mathbf{w}_l * \mathbf{z}_{(l-1)}\right), l \in\left[L_C\right], \\
\mathbf{z}_0 & =\bx
\end{aligned}
$$
where $m$ is the demension of $\bz_{(L-1)}$, $\sigma(z)$ is the ReLU function $\max \{z, 0\}$, * is the convolutional operation, and $\operatorname{pool}(\cdot)$ is the average pooling operation.  When $L_C=0$, this is an MLP. For output layer $l=L$, let $\bw_L:=\left(a_1, \cdots, a_m\right)^{\top}$. For fully-connected layer $l \in[L-1]-\left[L_C\right]$, we let $\bw_l:=\operatorname{vector}\left(\mathbf{A}_{l}\right)$. For convolution layer $l \in\left[L_C\right]$, we consider the structure Convolution $\rightarrow \operatorname{ReLU} \rightarrow$ Pooling, and denotes the weights as $\bw_l$.

\subsection{Architecture of GAN in Section~\ref{sec: Implications on deep generative models}}
\label{sec: GAN arch}
\textbf{The abstract form of GAN.} The architecture of GAN in Theorem~\ref{thm: GAN generalization bound} is consistent with that in Theorem 19,~\cite{DBLP:journals/jmlr/Liang21}. We denote by $\mathcal{F}=\left\{f_{\boldsymbol{\omega}}(\bx): \mathbb{R}^d \rightarrow \mathbb{R}\right\}$ the discriminator function space. Besides, we let $\mathcal{G}=\left\{g_{\boldsymbol{\theta}}(\bz): \mathbb{R}^d \rightarrow \mathbb{R}^d\right\}$ be the generator function space. The generator receives $\bz \sim \mathrm{unif}[0,1]^d$ as the random input. In reality, we estimate the parameters of GAN as
$$
\widehat{\boldsymbol{\theta}}_{m, n} \in \underset{\boldsymbol{\theta}: g_{\boldsymbol{\theta}} \in \mathcal{G}}{\arg \min } \max _{\boldsymbol{\omega}: f_{\boldsymbol{\omega}} \in \mathcal{F}}\left\{\widehat{\mathbb{E}}_n f_{\boldsymbol{\omega}}\left(g_{\boldsymbol{\theta}}(Z)\right)-\widehat{\mathbb{E}}_m f_{\boldsymbol{\omega}}(X)\right\},
$$
where $n$ and $m$ denote the number of simulated and target distribution samples, respectively. We just let $m = n$ in this paper.

\textbf{The architecture of the generator network.} The generator $g_{\boldsymbol{\theta}}$ is parametrized by a MLP:

$$
\begin{aligned}
\bh_0 & = \bz, \\
\bh_l & =\sigma_a\left(\bW_l \bh_{l-1}+\mathbf{b}_l\right), 0<l<L \\
\bx & =\bW_L \bh_{L-1}+ \mathbf{b}_L,
\end{aligned}
$$
where $h_l$ denotes the hidden units in the $l$-th layer, and $\bx$ is the final output of the MLP. The activation is leaky ReLU~\cite{DBLP:journals/corr/XuWCL15}.
$$
\sigma_a(t)=\max \{t, a t\}, \text { for some fixed } 0<a \leq 1
$$
The space for the generator weights is denoted by
$$
\Theta(d, L):=\left\{\boldsymbol{\theta}=\left(\bW_l \in \mathbb{R}^{d \times d}, \mathbf{b}_l \in \mathbb{R}^d, 1 \leq l \leq L\right) \mid \operatorname{rank}\left(\bW_l\right)=d, \forall 1 \leq l \leq L\right\} .
$$
Note that the $\bW_l$ is required to be full rank so that the generator transformation $g_\theta$ is invertible. The generator has the capacity to express complex distributions

\textbf{The architecture of the discriminator network.} We consider a discriminator network which includes feed-forward neural networks $f_{\boldsymbol{\omega}}$ that satisfies

$$
\begin{aligned}
\bh_1 & =\sigma_{1/a}\left(\bV_1 \bx+ \bc_1\right) \\
& \cdots \\
\bh_{L-1} & =\sigma_{1/a}\left(\bV_{L-1} \bh_{L-2}+ \bc_{L-1}\right) \\
q_{\boldsymbol{\omega}}(\bx) & :=\sum_{j=1}^{L-1} \sum_{i=1}^d \log (1 / a) 1_{h_{ji} \leq 0}+ c_L .
\end{aligned}
$$
The parameter space is defined as
$$
\Omega(d, L):=\left\{ \boldsymbol{\omega} = \left(\bV_l \in \mathbb{R}^{d \times d}, \bc_l \in \mathbb{R}^d, c_L \in \mathbb{R}, 1 \leq l \leq L-1\right) \mid \operatorname{rank}\left(\bV_l\right)=d, \forall 1 \leq l \leq L-1\right\}.
$$

Finally, the discriminator parameterized by $\boldsymbol{\omega} =\left(\boldsymbol{\omega}_1, \boldsymbol{\omega}_2\right)$, where $\boldsymbol{\omega}_1, \boldsymbol{\omega}_2 \in \Omega(d, L)$, is defined as
$$
f_{\boldsymbol{\omega}}(\bx)=q_{\boldsymbol{\omega}_1}(\bx)-q_{\boldsymbol{\omega}_2}(\bx) .
$$

\section{Proofs}

\subsection{Proof of Theorem~\ref{thm: main generalization bound}}
\label{proof: thm: main generalization bound}

\begin{proof}
We first list some moment inequalities which are important to this proof.

\begin{lemma}[Lemma 1,~\cite{DBLP:conf/colt/BousquetKZ20}]
\label{lemma: moment inequality}
	
If \( \| Y \|_{p} \leq \sqrt{p} a + pb \) for any \( p \geq 1\), then for any \( \delta \in (0, 1)\), with probability at least $1 - \delta$,
\[
|Y| \le e\left(a \sqrt{\log \left(\frac{e}{\delta}\right) } + b \log\left( \frac{e}{\delta}\right)\right) .
\]
\end{lemma}

\begin{lemma}[Lemma 2,~\cite{DBLP:conf/colt/BousquetKZ20}]
\label{lemma: boundeddiff}
Consider a function $f$ of independent random variables $X_1, \dots, X_n$ where \(X_i \in \mathcal{X}\). Suppose that for any \( i = 1, \dots, n \) and any \( x_1, \dots, x_n, x_i' \in \mathcal{X} \) it holds that
\begin{equation}
\label{boundedcond}
|f(x_1, \dots, x_n) - f(x_1, \dots, x_{i-1}, x_i', x_{i + 1}, \dots, x_n)| \leq \beta .
\end{equation}
Then, we have for any $p \ge 2$,
\[
\|f(X_1, \ldots, X_n) - \E f(X_1, \ldots, X_n)\|_{p} \le 2\sqrt{np}\beta \, .
\]
\end{lemma}

\begin{lemma}[Theorem 4,~\cite{DBLP:conf/colt/BousquetKZ20}]
\label{lemma: concentration bound}
Let $\bZ = (Z_{1}, \ldots, Z_n) $ be a vector of independent random variables each taking values in $\mathcal{Z}$, and let $g_{1}, \ldots, g_n$ be some functions $g_{i}: \mathcal{Z}^n \to \mathbb{R}$ such that the following holds for any $i \in [n]$:
\begin{itemize}
	\item $\bigl|\E[g_i(\bZ)| Z_i]\bigr| \le M$,
	\item $\E [g_i(\bZ) | \bZ^{\setminus i}] = 0$, 
	\item $g_i$ has a bounded difference $\beta$ with respect to all variables except the $i$-th variable, that is, for all $j \ne i$, $\bZ = (Z_1, \dots , Z_n)$ and $\bZ^j = (Z_1, \dots, Z_j', \dots, Z_n) \in \R^n$, we have $\abs{g_i(\bZ) - g_i(\bZ^j)} \leq \beta$.
\end{itemize}
Then, for any $p \ge 2$,
\[
\left\|\sum\limits_{i = 1}^n{g}_i(\bZ)\right\|_p \leq 12\sqrt{2} p n \beta  \log n  + 4M\sqrt{pn} .
\]
\end{lemma}

Now, we are ready to prove Theorem~\ref{thm: main generalization bound}. Formally, we need to bound $\textit{Gen-error} = \vert\sR_{\sD}(\sA(\Saug)) - \widehat{\sR}_{\Saug}(\sA(\Saug))\vert$. Recall that $\Daug(S)$ has been defined as the mixed distribution after augmentation, to derive such a bound, we first decomposed \textit{Gen-error} as 
\begin{align*}
\abs{\textit{Gen-error}} \leq \underbrace{\abs{\sR_{\sD}(\sA(\Saug)) - \sR_{\Daug(S)}(\sA(\Saug))}}_{\text{Distributions' divergence}} + \underbrace{\abs{ \sR_{\Daug(S)}(\sA(\Saug)) - \widehat{\sR}_{\Saug}(\sA(\Saug))}}_{\text{Generaliztion error w.r.t. mixed distribution, } \Phi(S, S_G) }.
\end{align*}
The distributions' divergence term in the right hand can be bounded by the divergence (e.g., $d_{\mathrm{TV}}, d_{\mathrm{KL}}$) between augmented distribution $\Daug(S)$ and the true distribution $\sD$. It is heavily dependent on the ability of the chosen generative model. It can be bounded as follows.

\begin{align*}
\abs{\sR_{\sD}(\sA(\Saug)) - \sR_{\Daug(S)}(\sA(\Saug))} &= \frac{m_G}{m_T} \abs{\sR_{\sD}(\sA(\Saug)) - \sR_{\DG}(\sA(\Saug))} \\
&= \frac{m_G}{m_T} \abs{\int_{\bz} \ell(\sA(\Saug), \bz) \left(\P_{\sD}(\bz) - \P_{\DG}(\bz) \right) d\bz} \\
&\leq \frac{m_G}{m_T} \int_{\bz} \abs{\ell(\sA(\Saug), \bz) \left(\P_{\sD}(\bz) - \P_{\DG}(\bz) \right)} d\bz \\
&\leq \frac{m_G}{m_T} M \int_{\bz} \abs{ \P_{\sD}(\bz) - \P_{\DG}(\bz)} d\bz \\
&\lesssim \frac{m_G}{m_T}M d_{\mathrm{TV}} \left(\sD , \DG \right).
\end{align*}

For the second term $\Phi(S, S_G)$, we note that classical stability bounds (e.g. Theorem~\ref{thm: classical stability bound}) can not be used directly, because points in $\Saug$ are drawn non-i.i.d.. In contrast, a core property of $\Saug$ is that $S$ satisfies i.i.d. assumption, and $S_G$ satisfies conditional i.i.d. assumption when $S$ is fixed. Inspired by this property, we furthermore decomposed this term and utilized sharp moment inequalities~\cite{boucheron2013concentration, DBLP:conf/colt/BousquetKZ20} to obtain an upper bound. Similarly to~\cite{DBLP:conf/colt/BousquetKZ20}, we bound the $L_p$ norm of $m_T\Phi(S, S_G)$, and then derive a concentration bound. We can write

\begin{align*}
\norm{m_T\Phi(S, S_G)}_p &= \norm{m_T\left(\sR_{\Daug(S)}(\sA(\Saug)) - \widehat{\sR}_{\Saug}(\sA(\Saug))\right)}_p \\
&= \norm{m_S\sR_{\sD}(\sA(\Saug)) + m_G\sR_{\DG}(\sA(\Saug)) - \sum_{\bz_i \in S} \ell(\sA(\Saug), \bz_i) - \sum_{\bz_i \in S_G} \ell(\sA(\Saug), \bz_i)}_p \\
&\leq \underbrace{\norm{m_S\sR_{\sD}(\sA(\Saug))  - \sum_{i=1}^{m_S} \ell(\sA(\Saug), \bz_i)}_p}_{\norm{\Phi_1(S, S_G)}_p} + \underbrace{\norm{m_G\sR_{\DG}(\sA(\Saug)) - \sum_{i=1}^{m_G} \ell(\sA(\Saug), \bz^G_i)}_p}_{\norm{\Phi_2(S, S_G)}_p}.
\end{align*}

We will bound $\norm{\Phi_1(S, S_G)}_p$ and $\norm{\Phi_2(S, S_G)}_p$ respectively. We note that for any function $f(S)$, if we have an bound $ \| f \|_{p}(S_V) \leq C$ for some $S_V \subseteq S$, then we have
\begin{equation}
\label{eqn: conditional Lp norm}
\| f \|_{p} = (\E \E[|f|^{p} \vert S_V ])^{1/p} \leq (\E [C^p ])^{1/p} \leq C.
\end{equation}

Fix $S$, then data in $S_G$ are independent. We use this property and lemma~\ref{lemma: concentration bound} to bound $\norm{\Phi_2}_p(S)$. We introduce functions $f_i(S_G)$ which play the same role as $g_i$s in Lemma~\ref{lemma: concentration bound}, as

\begin{equation*}
f_i(S_G) = \E_{\bz_i' \sim \DG} \left[ \E_{\bz \sim \DG} \ell(\sA(S \cup S_G^i), \bz) - \ell(\sA(S \cup S_G^i), \bz^G_i) \right],
\end{equation*}
where $\bz^G_i$ is the $i$-th data in $S_G$, and $S_G^i$ obtained by replacing $\bz^G_i$ by $\bz'_i$. We note that $|f_i| \leq M$, $\E [f_i | S_G^{\setminus i}] = 0$ and $f_i$ has a bounded difference $2\beta_{m_T}$ with respect to all variables except the $i$-th variable, which can be proved as follows.

\begin{align*}
|f_i|  &= \left| \E_{\bz_i' \sim \DG} \left[ \E_{\bz \sim \DG} \ell(\sA(S \cup S_G^i), \bz) - \ell(\sA(S \cup S_G^i), \bz^G_i) \right] \right| \\
&= \left| \E_{\bz_i' \sim \DG}  \E_{\bz \sim \DG} \left[\ell(\sA(S \cup S_G^i), \bz) - \ell(\sA(S \cup S_G^i), \bz^G_i) \right] \right| \\
&\leq \E_{\bz_i' \sim \DG}  \E_{\bz \sim \DG} \left|\ell(\sA(S \cup S_G^i), \bz) - \ell(\sA(S \cup S_G^i), \bz^G_i) \right| \\
&\leq \E_{\bz_i' \sim \DG}  \E_{\bz \sim \DG} [M]  = M,
\end{align*}

\begin{align*}
\E [f_i | S_G^{\setminus i}]  &= \E_{\bz^G_i \sim \DG} \left[ \E_{\bz_i' \sim \DG} \left[ \E_{\bz \sim \DG} \ell(\sA(S \cup S_G^i), \bz) - \ell(\sA(S \cup S_G^i), \bz^G_i) \right] | S_G^{\setminus i}\right]  \\
&= \E_{\bz_i' \sim \DG} \left[  \left[ \E_{\bz \sim \DG} \ell(\sA(S \cup S_G^i), \bz) - \E_{\bz^G_i \sim \DG} \ell(\sA(S \cup S_G^i), \bz^G_i) \right] | S_G^{\setminus i}\right]  \\
&= \E_{\bz_i' \sim \DG} \left[  0 | S_G^{\setminus i}\right] = 0,
\end{align*}

\begin{align*}
\abs{f_i(S_G) - f_i(S_G^j)}  &= \bigg| \E_{\bz_i' \sim \DG} \left[ \E_{\bz \sim \DG} \ell(\sA(S \cup S_G^i), \bz) - \ell(\sA(S \cup S_G^i), \bz^G_i) \right] \\
&\quad -\E_{\bz_i' \sim \DG} \left[ \E_{\bz \sim \DG} \ell(\sA(S \cup (S_G^j)^i), \bz) - \ell(\sA(S \cup (S_G^j)^i, \bz^G_i) \right] \bigg|  \\
&= \bigg| \E_{\bz_i' \sim \DG} \bigg[ \E_{\bz \sim \DG} \ell(\sA(S \cup S_G^i), \bz) - \ell(\sA(S \cup S_G^i), \bz^G_i)  \\
&\quad- \E_{\bz \sim \DG} \ell(\sA(S \cup (S_G^j)^i), \bz) + \ell(\sA(S \cup (S_G^j)^i, \bz^G_i) \bigg] \bigg|  \\
&\leq \left| \E_{\bz_i' \sim \DG} \E_{\bz \sim \DG} \left[\ell(\sA(S \cup S_G^i), \bz) - \ell(\sA(S \cup (S_G^j)^i), \bz) \right] \right| \\
&\quad + \left| \E_{\bz_i' \sim \DG}  \left[\ell(\sA(S \cup S_G^i), \bz^G_i) - \ell(\sA(S \cup (S_G^j)^i), \bz^G_i) \right] \right| \\
&\leq  \E_{\bz_i' \sim \DG} \E_{\bz \sim \DG} \left|\ell(\sA(S \cup S_G^i), \bz) - \ell(\sA(S \cup (S_G^j)^i), \bz) \right| \\
&\quad + \E_{\bz_i' \sim \DG}  \left|\ell(\sA(S \cup S_G^i), \bz^G_i) - \ell(\sA(S \cup (S_G^j)^i), \bz^G_i) \right| \\
&\leq  \beta_{m_T} + \beta_{m_T} = 2\beta_{m_T}.
\end{align*}

Therefore, for any fixed $S$, by Lemma~\ref{lemma: concentration bound}, for any $p \ge 2$, we have

\begin{align}
\label{eqn: sum f_i S_G}
\left\|\sum_{i = 1}^{m_G} {f}_i(S_G)\right\|_p \lesssim p m_G \beta_{m_T}  \log m_G  + M\sqrt{pm_G} .
\end{align}

We note the gap between $\Phi_2$ and $\sum_{i=1}^{m_G} f_i$ is small, then for any fixed $S$, we can bound $\norm{\Phi_2}_p(S)$ by (\ref{eqn: sum f_i S_G}) as follows.

\begin{align*}
\norm{\Phi_2}_p(S) &= \norm{m_G\sR_{\DG}(\sA(S \cup S_G)) - \sum_{i=1}^{m_G} \ell(\sA(S \cup S_G), \bz^G_i)}_p \\
&= \norm{\sum_{i=1}^{m_G} \left(\E_{\bz \sim \DG} \ell(\sA(S \cup S_G), \bz) -  \ell(\sA(S \cup S_G), \bz^G_i) \right)}_p \\
&\leq \norm{\sum_{i=1}^{m_G} \left(\E_{\bz_i' \sim \DG} \left[ \E_{\bz \sim \DG} \ell(\sA(S \cup S_G^i), \bz) - \ell(\sA(S \cup S_G^i), \bz^G_i) \right] \right)}_p +  \norm{2 m_G\beta_{m_T}}_p\\
&= \norm{\sum_{i = 1}^{m_G} {f}_i(S_G)}_p +  \norm{2 m_G\beta_{m_T}}_p\\
&\lesssim p m_G \beta_{m_T}  \log m_G  +  M\sqrt{pm_G} + 2m_G\beta_{m_T}\\
&\lesssim p m_G \beta_{m_T}  \log m_G  + M\sqrt{pm_G}.
\end{align*}

Therefore, by using (\ref{eqn: conditional Lp norm}), we have
\begin{equation}
\label{eqn: phi2}
\norm{\Phi_2(S, S_G)}_p \lesssim p m_G \beta_{m_T}  \log m_G  + M\sqrt{pm_G}.
\end{equation}

Now, we use a similar idea to bound $\norm{\Phi_1(S, S_G)}_p$. We decompose $\norm{\Phi_1(S, S_G)}_p$ as the following.

\begin{align*}
\norm{\Phi_1(S, S_G)}_p &= \norm{\Phi_1 - \E_{S_G \sim \sD_G^{m_G}(S)} \Phi_1 + \E_{S_G \sim \sD_G^{m_G}(S)} \Phi_1 - \E_{S_G \sim \sD_G^{m_G}} \Phi_1 + \E_{S_G \sim \sD_G^{m_G}} \Phi_1}_p\\
&\leq \underbrace{\norm{\Phi_1 - \E_{S_G \sim \sD_G^{m_G}(S)} \Phi_1}_p}_{\Delta_1} + \underbrace{\norm{\E_{S_G \sim \sD_G^{m_G}(S)} \Phi_1 }}_{\Delta_2},
\end{align*}
where $\sD_G = \E_S[\DG]$. We then bound each term and obtain a bound for $\norm{\Phi_1(S, S_G)}_p$. We note that $\Delta_1$ can be bounded by using Lemma~\ref{lemma: boundeddiff} and $\Delta_2$ can be bounded by using Lemma~\ref{lemma: concentration bound}.


To bound $\Delta_1$, we first fix $S$ and bound $\norm{\Phi_1 - \E_{S_G \sim \sD_G^{m_G}(S)} \Phi_1}_p(S)$. We use the conditional independence property of $S_G$ again. To use Lemma~\ref{lemma: boundeddiff}, we need to prove that $\Phi_1$ has the bounded difference with respect to $S_G$ when $S$ is fixed. We can write

\begin{align*}
&\abs{\Phi_1(S, S_G) - \Phi_1(S, S_G^i)} \\
&= \abs{m_S\sR_{\sD}(\sA(S \cup S_G))  - \sum_{i=1}^{m_S} \ell(\sA(S \cup S_G), \bz_i) - m_S\sR_{\sD}(\sA(S \cup S_G^i))  + \sum_{i=1}^{m_S} \ell(\sA(S \cup S_G^i), \bz_i)}\\
&\leq m_S\abs{\sR_{\sD}(\sA(S \cup S_G)) - \sR_{\sD}(\sA(S \cup S_G^i))} + \sum_{i=1}^{m_S} \abs{\ell(\sA(S \cup S_G), \bz_i) - \ell(\sA(S \cup S_G^i), \bz_i)}\\
&\leq m_S\beta_{m_T} + m_S\beta_{m_T} = 2 m_S\beta_{m_T}.
\end{align*}
Thus, by Lemma~\ref{lemma: boundeddiff}, we have 

\begin{equation}
\label{eqn: delta1}
\Delta_1 \leq 4 \sqrt{m_G p} m_S\beta_{m_T} \lesssim \sqrt{m_G p} m_S\beta_{m_T}.
\end{equation}



We now construct some functions and use Lemma~\ref{lemma: concentration bound} again to bound $\Delta_2$. We define $h_i(S)$ which play the same role as $g_i$s in Lemma~\ref{lemma: concentration bound}, as

\begin{equation*}
h_i(S) = \E_{\bz_i' \sim \sD} \E_{S_G \sim \sD_G^{m_G}(S^i)} \left[ \E_{\bz \sim \sD} \ell(\sA(S^i \cup S_G), \bz) - \ell(\sA(S^i \cup S_G), \bz_i) \right],
\end{equation*}
where $\bz_i$ is the $i$-th data in $S$, and $S^i$ obtained by replacing $\bz_i$ by $\bz'_i$. We note that $|h_i| \leq M$, $\E [h_i | S^{\setminus i}] = 0$ and $h_i$ has a bounded difference $2\beta_{m_T} + 2M \sT(m_S, m_G)$ with respect to all variables except the $i$-th variable, where $ \sT(m_S, m_G) =  \sup_i d_{\mathrm{TV}}\left(\sD_G^{m_G}(S), \sD_G^{m_G}(S^i) \right)$. These can be proved as follows.

\begin{align*}
|h_i|  &= \left| \E_{\bz_i' \sim \sD} \E_{S_G \sim \sD_G^{m_G}(S^i)} \left[ \E_{\bz \sim \sD} \ell(\sA(S^i \cup S_G), \bz) - \ell(\sA(S^i \cup S_G), \bz_i) \right] \right| \\
&= \left| \E_{\bz_i' \sim \sD} \E_{S_G \sim \sD_G^{m_G}(S^i)} \E_{\bz \sim \sD} \left[ \ell(\sA(S^i \cup S_G), \bz) - \ell(\sA(S^i \cup S_G), \bz_i) \right] \right| \\
&=  \E_{\bz_i' \sim \sD} \E_{S_G \sim \sD_G^{m_G}(S^i)} \E_{\bz \sim \sD} \left| \ell(\sA(S^i \cup S_G), \bz) - \ell(\sA(S^i \cup S_G), \bz_i) \right| \\
&\leq  M,
\end{align*}

\begin{align*}
\E [h_i | S^{\setminus i}]  &= \E_{\bz_i \sim \sD} \left[ \E_{\bz_i' \sim \sD} \E_{S_G \sim \sD_G^{m_G}(S^i)} \left[ \E_{\bz \sim \sD} \ell(\sA(S^i \cup S_G), \bz) - \ell(\sA(S^i \cup S_G), \bz_i) \right] | S^{\setminus i}\right]  \\
&= \E_{\bz_i' \sim \sD} \E_{S_G \sim \sD_G^{m_G}(S^i)} \left[  \left[ \E_{\bz \sim \sD} \ell(\sA(S^i \cup S_G), \bz)  - \E_{\bz_i \sim \sD} \ell(\sA(S^i \cup S_G), \bz_i) \right] | S^{\setminus i}\right]  \\
&= 0,
\end{align*}

\begin{align}
\abs{h_i(S) - h_i(S^j)}  &= \bigg| \E_{\bz_i' \sim \sD} \E_{S_G \sim \sD_G^{m_G}(S^i)} \left[ \E_{\bz \sim \sD} \ell(\sA(S^i \cup S_G), \bz) - \ell(\sA(S^i \cup S_G), \bz_i) \right] \nonumber \\
&\quad - \E_{\bz_i' \sim \sD} \E_{S_G \sim \sD_G^{m_G}((S^j)^i)} \left[ \E_{\bz \sim \sD} \ell(\sA((S^j)^i \cup S_G), \bz) - \ell(\sA((S^j)^i \cup S_G), \bz_i) \right] \bigg| \nonumber \\
&\leq \bigg| \E_{\bz_i' \sim \sD} \E_{S_G \sim \sD_G^{m_G}(S^i)} \left[ \E_{\bz \sim \sD} \ell(\sA(S^i \cup S_G), \bz) - \ell(\sA(S^i \cup S_G), \bz_i) \right] \nonumber \\
&\quad - \E_{\bz_i' \sim \sD} \E_{S_G \sim \sD_G^{m_G}(S^i)} \left[ \E_{\bz \sim \sD} \ell(\sA((S^j)^i \cup S_G), \bz) - \ell(\sA((S^j)^i \cup S_G), \bz_i) \right] \bigg| \label{eqn:h 1}  \\
&\quad + \bigg| \E_{\bz_i' \sim \sD} \E_{S_G \sim \sD_G^{m_G}(S^i)} \left[ \E_{\bz \sim \sD} \ell(\sA((S^j)^i \cup S_G), \bz) - \ell(\sA((S^j)^i \cup S_G), \bz_i) \right] \nonumber \\
&\quad - \E_{\bz_i' \sim \sD} \E_{S_G \sim \sD_G^{m_G}((S^j)^i)} \left[ \E_{\bz \sim \sD} \ell(\sA((S^j)^i \cup S_G), \bz) - \ell(\sA((S^j)^i \cup S_G), \bz_i) \right] \bigg| .\label{eqn:h 2}
\end{align}

We bound (\ref{eqn:h 1}) and (\ref{eqn:h 2}) respectively. The first can be bounded by using the property of uniform stability.

\begin{align*}
&\bigg| \E_{\bz_i' \sim \sD} \E_{S_G \sim \sD_G^{m_G}(S^i)} \left[ \E_{\bz \sim \sD} \ell(\sA(S^i \cup S_G), \bz) - \ell(\sA(S^i \cup S_G), \bz_i) \right] \\
&\quad - \E_{\bz_i' \sim \sD} \E_{S_G \sim \sD_G^{m_G}(S^i)} \left[ \E_{\bz \sim \sD} \ell(\sA((S^j)^i \cup S_G), \bz) - \ell(\sA((S^j)^i \cup S_G), \bz_i) \right] \bigg|  \\
&= \bigg| \E_{\bz_i' \sim \sD} \E_{S_G \sim \sD_G^{m_G}(S^i)} \big[ \E_{\bz \sim \sD} \ell(\sA(S^i \cup S_G), \bz) - \ell(\sA(S^i \cup S_G), \bz_i)  \\
&\quad - \E_{\bz \sim \sD} \ell(\sA((S^j)^i \cup S_G), \bz) + \ell(\sA((S^j)^i \cup S_G), \bz_i) \big] \bigg|  \\
&\leq \left| \E_{\bz_i' \sim \sD} \E_{S_G \sim \sD_G^{m_G}(S^i)} \E_{\bz \sim \sD} \left[\ell(\sA(S^i \cup S_G), \bz) - \ell(\sA((S^j)^i \cup S_G),\bz) \right] \right| \\
&\quad + \left| \E_{\bz_i' \sim \sD} \E_{S_G \sim \sD_G^{m_G}(S^i)}  \left[\ell(\sA(S^i \cup S_G), \bz_i) - \ell(\sA((S^j)^i \cup S_G), \bz_i) \right] \right| \\
&\leq  \E_{\bz_i' \sim \sD} \E_{S_G \sim \sD_G^{m_G}(S^i)} \E_{\bz \sim \sD} \left|\ell(\sA(S^i \cup S_G), \bz) - \ell(\sA((S^j)^i \cup S_G),\bz) \right| \\
&\quad + \E_{\bz_i' \sim \sD} \E_{S_G \sim \sD_G^{m_G}(S^i)}  \left| \ell(\sA(S^i \cup S_G), \bz_i) - \ell(\sA((S^j)^i \cup S_G), \bz_i)  \right| \\
&\leq  \beta_{m_T} + \beta_{m_T} = 2\beta_{m_T}.
\end{align*}

We denote $\ell(\sA((S^j)^i \cup S_G), \bz) - \ell(\sA((S^j)^i \cup S_G), \bz_i)$ by $B$ for convenience, then we have
\begin{align*}
&\bigg| \E_{\bz_i' \sim \sD} \E_{S_G \sim \sD_G^{m_G}(S^i)} \left[ \E_{\bz \sim \sD} \ell(\sA((S^j)^i \cup S_G), \bz) - \ell(\sA((S^j)^i \cup S_G), \bz_i) \right] \\
&\quad - \E_{\bz_i' \sim \sD} \E_{S_G \sim \sD_G^{m_G}((S^j)^i)} \left[ \E_{\bz \sim \sD} \ell(\sA((S^j)^i \cup S_G), \bz) - \ell(\sA((S^j)^i \cup S_G), \bz_i) \right] \bigg|\\
&= \bigg| \E_{\bz_i' \sim \sD} \E_{\bz \sim \sD} \E_{S_G \sim \sD_G^{m_G}(S^i)} \left[ \ell(\sA((S^j)^i \cup S_G), \bz) - \ell(\sA((S^j)^i \cup S_G), \bz_i) \right] \\
&\quad - \E_{\bz_i' \sim \sD} \E_{\bz \sim \sD} \E_{S_G \sim \sD_G^{m_G}((S^j)^i)} \left[  \ell(\sA((S^j)^i \cup S_G), \bz) - \ell(\sA((S^j)^i \cup S_G), \bz_i) \right] \bigg|\\
&= \bigg| \E_{\bz_i' \sim \sD} \E_{\bz \sim \sD} \E_{S_G \sim \sD_G^{m_G}(S^i)} \left[ B \right]  - \E_{\bz_i' \sim \sD} \E_{\bz \sim \sD} \E_{S_G \sim \sD_G^{m_G}((S^j)^i)} \left[B \right] \bigg|\\
&= \bigg| \E_{\bz_i' \sim \sD} \E_{\bz \sim \sD} \left[\E_{S_G \sim \sD_G^{m_G}(S^i)} \left[ B \right] - \E_{S_G \sim \sD_G^{m_G}((S^j)^i)} [B] \right] \bigg|\\
&\leq  \E_{\bz_i' \sim \sD} \E_{\bz \sim \sD} \bigg|\E_{S_G \sim \sD_G^{m_G}(S^i)} \left[ B \right] - \E_{S_G \sim \sD_G^{m_G}((S^j)^i)} [B]  \bigg|\\
& =\E_{\bz_i' \sim \sD} \E_{\bz \sim \sD} \abs{\int_{S_G} \left(\P(S_G|S^i) - \P(S_G|(S^j)^i)\right) B dS_G} \\
& \leq \E_{\bz_i' \sim \sD} \E_{\bz \sim \sD} 
 \left[\int_{S_G} \abs{\left(\P(S_G|S^i) - \P(S_G|(S^j)^i)\right) B} dS_G \right] \\
 & \leq M \E_{\bz_i' \sim \sD} \E_{\bz \sim \sD} 
 \left[\int_{S_G} \abs{\P(S_G|S^i) - \P(S_G|(S^j)^i) } dS_G \right] \\
& \leq 2M \sup_i d_{\mathrm{TV}}\left(\sD_G^{m_G}(S^i), \sD_G^{m_G}(S) \right) = 2M \sT(m_S, m_G).
\end{align*}

Therefore, $h_i$ has a bounded difference $2\beta_{m_T} + 2M \sT(m_S, m_G)$ with respect to all variables except the $i$-th variable. By Lemma~\ref{lemma: concentration bound}, we have

\begin{align}
\label{eqn: sum h_i S}
\left\|\sum_{i = 1}^{m_S} {h}_i(S)\right\|_p &\leq 12\sqrt{2} p m_S \left(2 \beta_{m_T} + 2 M \sT(m_S, m_G) \right)\log m_S  + 4M\sqrt{pm_S}\\
&\lesssim p m_S \left( \beta_{m_T} +  M \sT(m_S, m_G) \right)\log m_S  + M\sqrt{pm_S}.
\end{align}

We note the gap between $\Delta_2$ and $\norm{\sum_{i=1}^{m_S} h_i(S)}_p$ is small, then we can bound $\Delta_2$ by (\ref{eqn: sum h_i S}) as follows.

\begin{align}
\Delta_2 &= \norm{\E_{S_G \sim \sD_G^{m_G}(S)} \Phi_1}_p \nonumber \\
&= \norm{\E_{S_G \sim \sD_G^{m_G}(S)} \left[ m_S\sR_{\sD}(\sA(\Saug))  - \sum_{i=1}^{m_S} \ell(\sA(\Saug), \bz_i) \right]}_p \nonumber \\
&= \norm{\sum_{i=1}^{m_S} \E_{S_G \sim \sD_G^{m_G}(S)} \left[ m_S\sR_{\sD}(\sA(\Saug))  - \ell(\sA(\Saug), \bz_i) \right]}_p \nonumber \\
&\leq \norm{\sum_{i=1}^{m_S} \left(\E_{\bz_i' \sim \sD} \E_{S_G \sim \sD_G^{m_G}(S^i)} \left[ \E_{\bz \sim \sD} \ell(\sA(S^i \cup S_G), \bz) - \ell(\sA(S^i \cup S_G), \bz_i) \right] \right)}_p \\
& \quad +  \norm{2 m_S\beta_{m_T} + 2m_SM \sup_i d_{\mathrm{TV}}\left(\sD_G^{m_G}(S), \sD_G^{m_G}(S^i) \right)}_p \nonumber \\
&= \left\|\sum_{i = 1}^{m_S} {h}_i(S)\right\|_p +  \norm{2 m_S\beta_{m_T} + 2m_SM  \sT(m_S, m_G) }_p \nonumber \\
&\lesssim p m_S \left(\beta_{m_T} +  M \sT(m_S, m_G) \right)\log m_S  + M\sqrt{pm_S} \nonumber \\
&\quad +  m_S\beta_{m_T} + m_SM \sT(m_S, m_G) \nonumber \\
&\lesssim p m_S \left(\beta_{m_T} +  M \sT(m_S, m_G) \right)\log m_S  + M\sqrt{pm_S}. \label{eqn: delta3}
\end{align}

Combine (\ref{eqn: delta1}) and (\ref{eqn: delta3}), we have

\begin{align}
\norm{\Phi_1(S, S_G)}_p 
&\lesssim \sqrt{m_G p} m_S\beta_{m_T} + p m_S \left(\beta_{m_T} +  M \sT(m_S, m_G) \right)\log m_S  + M\sqrt{pm_S} \nonumber \\
&= \sqrt{p} \left(M\sqrt{m_S} + \sqrt{m_G} m_S\beta_{m_T} \right) + p m_S \left(\beta_{m_T} +  M \sT(m_S, m_G) \right)\log m_S \label{eqn: phi1}
\end{align}

In addition, by (\ref{eqn: phi1}) and (\ref{eqn: phi2}), we have

\begin{align*}
\norm{m_T\Phi(S, S_G)}_p &\lesssim 
\sqrt{p} \left(M\sqrt{m_S} + M\sqrt{m_G} + \sqrt{m_G} m_S\beta_{m_T} \right) \\
& \quad +  p \left(m_S \beta_{m_T}\log m_S + m_G \beta_{m_T}\log m_G + m_S \log m_S M \sT(m_S, m_G) \right).
\end{align*}

By Lemma~\ref{lemma: moment inequality}, we can bound the generalization error w.r.t. mixed distribution \abs{\Phi(S, S_G)} = \abs{\sR_{\Daug(S)}(\sA(\Saug)) - \widehat{\sR}_{\Saug}(\sA(\Saug))} as follows.

\begin{align*}
&\abs{\sR_{\Daug(S)}(\sA(\Saug)) - \widehat{\sR}_{\Saug}(\sA(\Saug))}\\
&\lesssim  \frac{M(\sqrt{m_S} + \sqrt{m_G}) + m_S\sqrt{m_G}\beta_{m_T} }{m_T} \sqrt{\log \left(\frac{1}{\delta}\right)}\\
&+\frac{\beta_{m_T} \left(m_S \log m_S + m_G \log m_G\right) + m_S \log m_S M \sT(m_S, m_G) }{m_T} \log \left(\frac{1}{\delta}\right).
\end{align*}

Finally, we conclude that

\begin{align*}
&\abs{\sR_{\sD}(\sA(\Saug)) - \widehat{\sR}_{\Saug}(\sA(\Saug))}\\
&\lesssim \frac{m_G}{m_T}M d_{\mathrm{TV}} \left(\sD , \DG \right) + \frac{M(\sqrt{m_S} + \sqrt{m_G}) + m_S\sqrt{m_G}\beta_{m_T} }{m_T} \sqrt{\log \left(\frac{1}{\delta}\right)}\\
&\quad + \frac{\beta_{m_T} \left(m_S \log m_S + m_G \log m_G\right) + m_S \log m_S M \sT(m_S, m_G) }{m_T} \log \left(\frac{1}{\delta}\right)\\
&\lesssim \frac{m_G}{m_T}M d_{\mathrm{TV}} \left(\sD , \DG \right) + \frac{M(\sqrt{m_S} + \sqrt{m_G}) + m_S\sqrt{m_G}\beta_{m_T} }{m_T} \sqrt{\log \left(\frac{1}{\delta}\right)}\\
&\quad +\frac{\beta_{m_T} \left(m_S \log m_S + m_G \log m_G\right) + m_S \log m_S M \sT(m_S, m_G) }{m_T} \log \left(\frac{1}{\delta}\right),
\end{align*}
which completes the proof.
 
\end{proof}

\subsection{Proof of Theorem~\ref{thm: bGMM generalization bound}}
\label{proof: thm bGMM generalization bound}

We need to bound terms $M$, $\beta_{m_T}$, $d_{\mathrm{TV}} \left(\sD, \DG \right)$ and $\sT(m_S, m_G)$ in Theorem~\ref{thm: main generalization bound}. For $M$ (Lemma~\ref{lemma: M}) and $\beta_{m_T}$ (Lemma~\ref{lemma: beta}), we mainly use the boundedness of the multivariate Gaussian variable with high probability (Lemma~\ref{lemma: gaussian bounded}). In addition, we bound $d_{\mathrm{TV}} \left(\sD, \DG \right)$ (Lemma~\ref{lemma: kl}) by discussing the distance between the estimated parameters and the true parameters of bGMM. Besides, the concentration property of $\sT(m_S, m_G)$ (Lemma~\ref{lemma: tau}) can be induced by the preceding discussion.

\begin{lemma}["Boundedness" of multivariate Gaussian distribution]
\label{lemma: gaussian bounded}
Let $\bX = (X_1, \dots, X_d)$ be a $d$-dimension isotropic Gaussian random variable, which satisfies $\Vert\boldsymbol{\mu}\Vert_2 = 1$ and $\sigma_i^2 = \sigma^2$ for any $i \in \{1, \dots, d\}$. For any $\delta \in (0,1)$, with probability at least $1-\delta$, it holds that
\begin{equation*}
\norm{\bX}_2 \lesssim \sigma \sqrt{d 
 + \log(\frac{1}{\delta})}.
\end{equation*}
\end{lemma}

\begin{proof}
The proof idea is to bound the distance between $\norm{\bX}_2^2$ and its expectation with high probability. Let $\bZ$ be the standard $d$-dimension isotropic Gaussian random variable, we have
\begin{align*}
&\P \left(\abs{\frac{\norm{\bX}_2^2}{d} - \sigma^2 - \frac{1}{d}} \ge \epsilon\right) \\
&= \P\left(\abs{\frac{1}{d} \sum_{i = 1}^d \left(X_i^2 - \sigma^2 - \mu_i^2 \right)} \ge \epsilon \right)\\
&= \P\left(\abs{\frac{1}{d} \sum_{i = 1}^d \left((\sigma Z_i + \mu_i)^2 - \sigma^2 - \mu_i^2 \right)} \ge \epsilon \right) \\
&= \P\left(\abs{\frac{1}{d} \sum_{i = 1}^d \left( \sigma^2(Z_i^2 - 1) + 2\sigma \mu_i Z_i \right)} \ge \epsilon \right)\\
&\leq \P\left(\abs{\frac{1}{d} \sum_{i = 1}^d \left( \sigma^2(Z_i^2 - 1) \right)}  + \abs{\frac{1}{d} \sum_{i = 1}^d \left( 2\sigma \mu_i Z_i\right)} \ge \epsilon \right) \\
&\leq \P\left(\abs{\frac{1}{d} \sum_{i = 1}^d \left( \sigma^2(Z_i^2 - 1) \right)}  \ge \frac{\epsilon}{2} \cup \abs{\frac{1}{d} \sum_{i = 1}^d \left( 2\sigma \mu_i Z_i\right)} \ge \frac{\epsilon}{2} \right)\\
&\leq \P\left(\abs{\frac{1}{d} \sum_{i = 1}^d \left( \sigma^2(Z_i^2 - 1) \right)}  \ge \frac{\epsilon}{2}\right) + \P\left(\abs{\frac{1}{d} \sum_{i = 1}^d \left( 2\sigma \mu_i Z_i\right)} \ge \frac{\epsilon}{2} \right)\\
&= \P\left(\abs{\frac{1}{d} \sum_{i = 1}^d \left( Z_i^2 - 1 \right)}  \ge \frac{\epsilon}{2\sigma^2}\right) + \P\left(\abs{\frac{1}{d} \sum_{i = 1}^d  \mu_i Z_i} \ge \frac{\epsilon}{4\sigma} \right).
\end{align*}

We bound each of the two terms respectively. For the first term, we note that $Z_i^2$ obeys $\chi^2(1)$ distribution and is a sub-exponential random variable, so it can be bounded by using Bernstein’s inequality (e.g., Proposition 2.9,~\cite{wainwright2019high}). By Example 2.8 in~\cite{wainwright2019high}, for any $\lambda \in (0, 1/4),$ we have

\begin{align*}
\E \left[\exp\left(\lambda(Z_i^2 - 1) \right) \right] = \frac{\exp(-\lambda)}{\sqrt{1 - 2\lambda}} \leq \exp(2\lambda^2).
\end{align*}
In addition, through Bernstein’s inequality, we have
\begin{align*}
&\P\left(\abs{\frac{1}{d} \sum_{i = 1}^d \left( Z_i^2 - 1 \right)}  \ge \frac{\epsilon}{2\sigma^2}\right) \leq \begin{cases} 
2\exp(-\frac{d\epsilon^2}{32\sigma^4}) & \text { if }  0 \leq \epsilon \leq 2 \sigma^2,\\ 
2\exp(-\frac{d\epsilon}{32\sigma^2}) & \text { if } \epsilon > 2 \sigma^2.
\end{cases}
\end{align*}

For the second term, we bound it directly by using Hoeffding's inequality (e.g., Proposition 2.5,~\cite{wainwright2019high}).

\begin{align*}
\P\left(\abs{\frac{1}{d} \sum_{i = 1}^d   \mu_i Z_i} \ge \frac{\epsilon}{4\sigma} \right) \leq 2\exp(-\frac{d\epsilon^2}{32 \sigma^4 \sum_{i=1}^d \mu_i^2}) = 2\exp(-\frac{d\epsilon^2}{32\sigma^4}).
\end{align*}

Therefore, for any $\epsilon \leq 2\sigma^2$, we have
\begin{align*}
\P \left(\abs{\norm{\bX}_2^2 - \sigma^2d - 1} \ge d\epsilon\right) = 
\P \left(\abs{\frac{\norm{\bX}_2^2}{d} - \sigma^2 - \frac{1}{d}} \ge \epsilon\right) \leq 4\exp(-\frac{d\epsilon^2}{32\sigma^4}).
\end{align*}
Let $4\exp(-\frac{d\epsilon^2}{32\sigma^4}) = \delta$, then with probability at least $1-\delta$, it holds that
\begin{align*}
\norm{\bX}_2^2 \leq \sigma^2d + 1 + d \sigma^2 \sqrt{\frac{32}{d} \log(\frac{4}{\delta})} 
\lesssim \sigma^2 \left(d  + \sqrt{ d \log(\frac{1}{\delta})}\right)
\end{align*}
which means that
\begin{align*}
\norm{\bX}_2 \lesssim \sigma \sqrt{d 
 + \sqrt{d \log(\frac{1}{\delta})}} \leq \sigma \sqrt{d 
 + \frac{1}{2} {d} + \frac{1}{2} \log(\frac{1}{\delta})} \lesssim \sigma \sqrt{d 
 + \log(\frac{1}{\delta})} .
\end{align*}

Similarly, for any $\epsilon > 2\sigma^2$, we have
\begin{align*}
\P \left(\abs{\frac{\norm{\bX}_2^2}{d} - \sigma^2 - \frac{1}{d}} \ge \epsilon\right) &\leq 2\exp(-\frac{d\epsilon}{32\sigma^2}) + 2\exp(-\frac{d\epsilon^2}{32\sigma^4})\\
& \leq 2\exp(-\frac{d\epsilon}{32\sigma^2}) + 2\exp(-\frac{d\epsilon}{16\sigma^2})\\
& \leq 4\exp(-\frac{d\epsilon}{32\sigma^2}).
\end{align*}
Let $4\exp(-\frac{d\epsilon}{32\sigma^2}) = \delta$, then with probability at least $1-\delta$, it holds that
\begin{align*}
\norm{\bX}_2^2 \leq \sigma^2d + 1 + d \sigma^2 \frac{32}{d} \log(\frac{4}{\delta}) \lesssim \sigma^2 \left(d 
 +  \log(\frac{1}{\delta})\right),
\end{align*}
which also implies
\begin{align*}
\norm{\bX}_2 \lesssim \sigma \sqrt{d 
 +  \log(\frac{4}{\delta})} \leq \sigma \sqrt{d 
 + \log(\frac{1}{\delta})} .
\end{align*}
The proof is completed.

\end{proof}

Based on the "boundedness" of multivariate Gaussian distribution, we can bound $M$, $\beta_{m}$, $d_{\mathrm{TV}}(\DG, \sD_G)$ and $\sT(m_S, m_G)$, respectively. They are listed as the following.

\begin{lemma}[Concentration bound for $M$]
\label{lemma: M}
For any $\delta \in (0,1)$, with probability at least $1-\delta$, it holds that
\begin{equation*}
\abs{\ell(\sA(S), \bz)} \lesssim d + \log(\frac{m}{\delta}).
\end{equation*}
\end{lemma}

\begin{proof}
Given a set $S = \{(\bx_1,y_1), \dots, (\bx_m, y_m)\}$ and $\bz$ sampled from binary mixture Gaussian distribution, by Lemma~\ref{lemma: gaussian bounded}, we know that for any $\delta \in (0,1)$, with probability at least $1-\delta$,
\begin{align*}
\max_i \norm{\bx_i}_2 \lesssim \sigma \sqrt{d 
 + \log(\frac{m+1}{\delta})}.
\end{align*}
Under this condition, we have 
\begin{align*}
&\abs{\ell(\sA(S), \bz)} \\
&= \abs{\frac{1}{2 \sigma^2}(\bx-y \boldsymbol{\theta})^{\top}(\bx-y \boldsymbol{\theta})} \\
&= \frac{1}{2 \sigma^2}\abs{\bx^{\top}\bx 
 - 2y \bx^{\top}\theta + \theta^{\top}\theta} \\
 &\leq \frac{1}{2 \sigma^2} \left( \abs{\bx^{\top}\bx} + 2\abs{ \bx^{\top}\theta} + \abs{\theta^{\top}\theta} \right) \\
 &\leq \frac{1}{2 \sigma^2} \left( \Vert\bx\Vert_2^2 + 2\Vert\bx\Vert_2 \Vert\theta\Vert_2 + \Vert\theta\Vert_2^2 \right)\\
  &= \frac{1}{2 \sigma^2} \left( \Vert\bx\Vert_2^2 + 2\Vert\bx\Vert_2 \Vert\frac{1}{m} \sum_{i=1}^m y_i\bx_i\Vert_2 + \Vert\frac{1}{m} \sum_{i=1}^m y_i\bx_i\Vert_2^2 \right)\\
&\leq \frac{1}{2 \sigma^2} \left( \Vert\bx\Vert_2^2 + 2\frac{1}{m} \sum_{i=1}^m \Vert\bx\Vert_2  \Vert \bx_i\Vert_2 + \left(\frac{1}{m} \sum_{i=1}^m  \Vert \bx_i\Vert_2\right)^2 \right)\\
&\lesssim \frac{1}{2 \sigma^2} \left(  \sigma^2\left( d + \log(\frac{m+1}{\delta}) \right) + \frac{2}{m} \sum_{i=1}^m \sigma^2\left( d + \log(\frac{m+1}{\delta}) \right)  + \left(\frac{1}{m} \sum_{i=1}^m \sigma \sqrt{d + \log(\frac{m+1}{\delta})} \right)^2 \right)\\
&= \frac{1}{2 \sigma^2} 4 \sigma^2\left( d + \log(\frac{m+1}{\delta}) \right) = 2\left( d + \log(\frac{m+1}{\delta}) \right)\\
&\lesssim  d + \log(\frac{m}{\delta}).
\end{align*}
\end{proof}

\begin{lemma}[Concentration bound for $\beta_m$]
\label{lemma: beta}
For any $\delta \in (0,1)$, with probability at least $1-\delta$, it holds that
\begin{equation*}
\abs{\ell(\sA(S), \bz) - \ell(\sA(S^i), \bz)} \lesssim \frac{1}{m} \left(d + \log(\frac{m}{\delta}) \right).
\end{equation*}
\end{lemma}

\begin{proof}
Given $m + 2$ samples $S$, $\bz$ and $\bz_i'$ randomly sampled from binary mixture Gaussian distribution, for any $\delta \in (0,1)$, with probability at least $1-\delta$, we have
\begin{align*}
&\abs{\ell(\sA(S), \bz) - \ell(\sA(S^i), \bz)} \\
&= \abs{\frac{1}{2 \sigma^2}(\bx-y \boldsymbol{\theta})^{\top}(\bx-y \boldsymbol{\theta}) - \frac{1}{2 \sigma^2}(\bx-y \boldsymbol{\theta'})^{\top}(\bx-y \boldsymbol{\theta'})} \\
&= \frac{1}{2 \sigma^2}\abs{
 2y \left(\bx^{\top}\theta' -\bx^{\top}\theta \right) + \theta^{\top}\theta - \theta'^{\top}\theta'} \\
 &= \frac{1}{2 \sigma^2}\abs{
 2y \left(\bx^{\top}\theta' -\bx^{\top}\theta \right) + (\theta + \theta')^{\top}(\theta - \theta')}\\
 &\leq \frac{1}{2 \sigma^2} \left( 2\abs{ \left(\bx^{\top}(\theta' -\theta) \right)} + \abs{(\theta + \theta')^{\top}(\theta - \theta')} \right) \\
&\leq \frac{1}{2 \sigma^2} \left( 2\Vert\bx \Vert_2 \Vert\theta' -\theta\Vert_2  + \Vert\theta + \theta'\Vert_2 \Vert\theta - \theta'\Vert_2 \right)\\
&= \frac{1}{2 \sigma^2} \left( 2\Vert\bx \Vert_2 + \Vert\theta + \theta'\Vert_2 \right) \Vert\theta' -\theta\Vert_2 \\
&= \frac{1}{2 \sigma^2} \left( 2\Vert\bx \Vert_2 + \Vert\theta + \theta'\Vert_2 \right) \Vert\frac{1}{m} (y_i\bx_i - y_i'\bx_i')\Vert_2 \\
&\leq \frac{1}{2 m\sigma^2} \left( 2\Vert\bx \Vert_2 + \Vert\theta\Vert_2 + \Vert\theta'\Vert_2 \right) \left(\Vert \bx_i\Vert_2 + \Vert \bx_i'\Vert_2\right)\\
&\lesssim \frac{8}{2 m\sigma^2} \sigma^2 \left(d + \log(\frac{m + 2}{\delta}) \right)\\
&\lesssim \frac{4}{m} \left(d + \log(\frac{m}{\delta}) \right) \lesssim \frac{1}{m} \left(d + \log(\frac{m}{\delta}) \right).
\end{align*}
\end{proof}

\begin{lemma}[Concentration bound for $d_{\mathrm{TV}}(\sD ,\DG)$]
\label{lemma: kl}
With high probability at least $1-\delta$, it holds that
\begin{equation*}
d_{\mathrm{TV}}(\sD ,\DG) \lesssim \max\left(1, \sqrt{\frac{d}{m}\log\left(\frac{d}{\delta} \right)}\right).
\end{equation*}
\end{lemma}

The idea of the proof of Lemma~\ref{lemma: kl} built upon the estimation for Gaussian distribution. As the sample size increases, parameters can be estimated more accurately, which leads to a smaller distance between the estimated and true Gaussian distributions. The concentration bound of the estimated parameters can be inscribed by the following lemma.

\begin{lemma}
\label{lemma: gaussian estiamtion}
Let $m = O\left(\frac{1}{\epsilon^2} \log\left(\frac{d}{\delta}\right)\right)$, then with high probability at least $1-\delta$, for any $i \in \{1, \dots, d\}$, it holds that
\begin{align*}
\abs{\frac{\widehat{\sigma^2_i}}{\sigma^2} - 1} \leq \epsilon, \quad \frac{\abs{\widehat{\mu_{yi}} - \mu_{yi}}} {\sigma} \leq \epsilon.
\end{align*}
\end{lemma}

\begin{proof}
Let $\epsilon \leq 1/4$, and $m_y$ be the number of samples from category $y$. By Hoeffding's inequality (Proposition 2.5,~\cite{wainwright2019high}), we have

\begin{equation*}
\P\left(\abs{m_y - \frac{m}{2}} \ge m\epsilon \right) \leq 2\exp(-\frac{m^2 \epsilon^2}{2m(1/2)^2}) = 2\exp(-2m\epsilon^2) = \delta_1,
\end{equation*}
which means $m_y \ge m/2 - \epsilon m \ge m/4$, and $m_y \leq m/2 + \epsilon m \leq 3m/4$. We can bound $\widehat{\sigma}^2_i$ and $\widehat{\mu}_{yi}$ based on the concentration property of $m_y$. In terms of $\widehat{\mu}_{yi}$, give a fixed $m_y$, we can write

\begin{align*}
\P\left(\frac{\abs{\widehat{\mu_{yi}} - \mu_{yi}}} {\sigma} \ge \epsilon \mid m_y\right) &= \P\left( \frac{1}{\sigma} \abs{\frac{\sum_{y_i = y}x_i}{m_y} - \mu_{yi}} \ge \epsilon \right) \\
&= \P\left(\abs{\sum_{y_i = y}x_i - m_y\mu_{yi}} \ge \sigma m_y\epsilon \right) \\
&\leq \exp\left(-\frac{\sigma^2 m_y^2 \epsilon^2}{2m_y \sigma^2} \right) = \exp\left(-\frac{m_y \epsilon^2}{2} \right).
\end{align*}
Furthermore, by the law of total probability, we have

\begin{align*}
&\P\left(\frac{\abs{\widehat{\mu_{yi}} - \mu_{yi}}} {\sigma} \ge \epsilon\right) \\
&= \P\left(\frac{\abs{\widehat{\mu_{yi}} - \mu_{yi}}} {\sigma} \ge \epsilon \mid m_y \ge m/2 - \epsilon m \right)\P(m_y \ge m/2 -\epsilon m) \\
&\quad + \P\left(\frac{\abs{\widehat{\mu_{yi}} - \mu_{yi}}} {\sigma} \ge \epsilon \mid m_y \leq m/2 -\epsilon m \right)\P(m_y \leq m/2 - \epsilon m)\\
&\leq \exp\left(-\frac{(m/4) \epsilon^2}{2} \right) + \delta_1 = \exp\left(-\frac{m \epsilon^2}{8} \right) + \delta_1 = \delta_2.
\end{align*}

For the estimation of $\widehat{\sigma}^2_i$, we can obtain its concentration bound in a similar way.

\begin{align*}
\P\left(\abs{\frac{\widehat{\sigma^2_i}}{\sigma^2} - 1} \ge \epsilon \mid m_y \right) &= \P\left(\abs{\sum_y \frac{m_y}{m\sigma^2} \frac{\sum_{y_i=y}(x_i - \widehat{\mu}_{yi})^2}{m_y - 1} - 1} \ge \epsilon\right) \\
&= \P\left(\abs{\sum_y \frac{m_y}{m} \left(\frac{\sum_{y_i=y}(x_i - \widehat{\mu}_{yi})^2}{(m_y - 1)\sigma^2} - 1\right)} \ge \epsilon\right) \\
&\leq \P\left(\sum_y\abs{\frac{m_y}{m} \left(\frac{\sum_{y_i=y}(x_i - \widehat{\mu}_{yi})^2}{(m_y - 1)\sigma^2} - 1\right)} \ge \epsilon\right) \\
&\leq \P\left(\cup_{y=\{-1,1\}} \abs{\frac{m_y}{m} \left(\frac{\sum_{y_i=y}(x_i - \widehat{\mu}_{yi})^2}{(m_y - 1)\sigma^2} - 1\right)} \ge \epsilon/2 \right) \\
&\leq \sum_y \P\left( \abs{\frac{m_y}{m} \left(\frac{\sum_{y_i=y}(x_i - \widehat{\mu}_{yi})^2}{(m_y - 1)\sigma^2} - 1\right)} \ge \epsilon/2 \right) \\
&\leq \sum_y \P\left( \abs{\frac{m_y}{m} \left(\frac{\sum_{y_i=y}(x_i - \widehat{\mu}_{yi})^2}{\sigma^2} - (m_y - 1)\right)} \ge (m_y - 1)\epsilon/2 \right) \\
&= \sum_y \P\left( \abs{\frac{\sum_{y_i=y}(x_i - \widehat{\mu}_{yi})^2}{\sigma^2} - (m_y - 1)} \ge \frac{(m_y - 1)m}{2m_y}\epsilon \right) \\
&= \sum_y \P\left( \abs{\chi^2(m_y-1) - (m_y - 1)} \ge \frac{(m_y - 1)m}{2m_y}\epsilon \right) \\
&= \sum_y \P\left( \abs{\frac{1}{m_y-1}\sum_{i=1}^{m_y-1} \chi^2(1) - 1} \ge \frac{m}{2m_y}\epsilon \right) \\
&\leq \sum_y 2\exp\left(-\frac{m_y-1}{8} (\frac{m}{2m_y}\epsilon)^2 \right) \quad \text{(Bernstein's inequality)} \\
&= \sum_y 2\exp\left(-\frac{(m_y-1)m^2\epsilon^2}{32m_y^2} \right)
\end{align*}
Without loss of generality, we assume that $m \ge 8$, then by the law of total probability, it holds that

\begin{align*}
\P\left(\abs{\frac{\widehat{\sigma^2_i}}{\sigma^2} - 1} \ge \epsilon \right) &= \P\left(\abs{\frac{\widehat{\sigma^2_i}}{\sigma^2} - 1} \ge \epsilon \mid \abs{m_y - m/2} \leq \epsilon m\right) \P(\abs{m_y - m/2} \leq \epsilon m) \\
&+ \P\left(\abs{\frac{\widehat{\sigma^2_i}}{\sigma^2} - 1} \ge \epsilon \mid \abs{m_y - m/2} \ge \epsilon m\right) \P(\abs{m_y - m/2} \ge \epsilon m) \\
&\leq \sum_y 2\exp\left(-\frac{(m_y-1)m^2\epsilon^2}{32m_y^2} \mid \frac{1}{4}m \leq m_y \leq \frac{3}{4}m \right) + \delta_1\\
&\leq \sum_y 2\exp\left(-\frac{(3m/4-1)m^2\epsilon^2}{32(3m/4)^2}  \right) + \delta_1 \quad \text{($\frac{x-1}{x^2}$ decreases when $x \ge 2$)} \\
&\leq 4 \exp\left(-\frac{m\epsilon^2}{36}  \right) + \delta_1 = \delta_3
\end{align*}

We can conclude that

\begin{align*}
&\P\left( \cup_{i=1}^d \cup_y \frac{\abs{\widehat{\mu_{yi}} - \mu_{yi}}} {\sigma} \ge \epsilon \cup \cup_{i=1}^d\abs{\frac{\widehat{\sigma^2_i}}{\sigma^2} - 1} \ge \epsilon \right)\\
&= 2d\delta_2 + d\delta_3\\
&= 2d\delta_1 + 2d\exp\left(-\frac{m \epsilon^2}{8} \right) + d\delta_1 + 8d \exp\left(-\frac{m\epsilon^2}{36}  \right)\\
&= 6d\exp(-2m\epsilon^2) + 2d\exp\left(-\frac{m \epsilon^2}{8} \right) + 8d \exp\left(-\frac{m\epsilon^2}{36}  \right)\\
&\leq 16d \exp\left(-\frac{m\epsilon^2}{36}  \right)
\end{align*}
Equivalently, when $m = \frac{36}{\epsilon^2} \log\left(\frac{16d}{\delta} \right) = O\left( \frac{1}{\epsilon^2} \log\left(\frac{d}{\delta} \right) \right)$, for any $\delta \in (0,1)$, with probability at least $1-\delta$, for any $i \in \{1, \dots, d\}$, we have

\begin{align*}
\abs{\frac{\widehat{\sigma^2_i}}{\sigma^2} - 1} \leq \epsilon, \quad \frac{\abs{\widehat{\mu_{yi}} - \mu_{yi}}} {\sigma} \leq \epsilon,
\end{align*}
which completes the proof of Lemma~\ref{lemma: gaussian estiamtion}.
\end{proof}

Based on the Lemma~\ref{lemma: gaussian estiamtion}, we can prove Lemma~\ref{lemma: kl} as follows.

\begin{proof}
Without loss of generality, we let $m = O\left(\frac{1}{\epsilon^2} \log\left(\frac{d}{\delta}\right)\right)$ as that in Lemma~\ref{lemma: gaussian estiamtion}. We can bound $d_{\mathrm{KL}}(\DG \Vert \sD)$ as follows.

\begin{align*}
&d_{\mathrm{KL}}(\DG \Vert \sD) \\
&= \int p_G(\bx,y)\log \frac{p_G(\bx,y)}{p(\bx,y)}\\
&= \int p_G(\bx,y)\log \frac{p_G(\bx\mid y)p_G(y)}{p(\bx\mid y)p(y)}\\
&= \int p_G(\bx,y)\log \frac{p_G(\bx\mid y)}{p(\bx\mid y)} & \text{($p_G(y) = p(y)$)}\\
&= \int_y p_G(y) \int_x p_G(\bx\mid y)\log \frac{p_G(\bx\mid y)}{p(\bx\mid y)} \\
&= \sum_{y} \frac{1}{2} \int_x p_G(\bx\mid y)\log \frac{p_G(\bx\mid y)}{p(\bx\mid y)} \\
&= \sum_{y} \frac{1}{2} \sum_{i=1}^d \frac{1}{2} \left(\frac{\widehat{\sigma^2_i}}{\sigma^2} - 1  - \log\left( \frac{\widehat{\sigma^2_i}}{\sigma^2}\right) + \frac{(\widehat{\mu_{yi}} - \mu_{yi})^2}{\sigma^2}\right)\\
&\leq \sum_{y} \frac{1}{2} \sum_{i=1}^d \frac{1}{2} \left( \left(\frac{\widehat{\sigma^2_i}}{\sigma^2}- 1\right)^2  + \frac{(\widehat{\mu_{yi}} - \mu_{yi})^2}{\sigma^2}\right) & \text{($x - log(x + 1) \leq x^2$, $\abs{x} \leq 1/2$)}\\
&\leq \sum_{y} \frac{1}{2} \sum_{i=1}^d \frac{1}{2} \left( \epsilon^2  + \epsilon^2 \right) = d\epsilon^2 \lesssim \frac{d}{m}\log\left(\frac{d}{\delta} \right).  & \text{(Lemma~\ref{lemma: gaussian estiamtion})}
\end{align*}

Finally, by the Pinsker’s inequality (such as,~\cite{DBLP:journals/tit/SasonV16}), we have 

\begin{equation*}
d_{\mathrm{TV}}(\sD ,\DG) \leq \max\left(1, \sqrt{ 2\log2 d_{\mathrm{KL}}(\DG, \sD)}\right) \lesssim \max\left(1, \sqrt{\frac{d}{m}\log\left(\frac{d}{\delta} \right)}\right),
\end{equation*}
which completes the proof of Lemma~\ref{lemma: kl}.
\end{proof}

\begin{lemma}[Concentration bound for $\sT(m_S, m_G)$]
\label{lemma: tau}
Let $\delta$ in Lemma~\ref{lemma: kl} be $\delta_1$, and $\delta$ in Lemma~\ref{lemma: gaussian bounded} be $\delta_2$, then With probability at least $1-\delta_1 - \delta_2$, it holds that
\begin{equation*}
\sT(m_S, m_G) \lesssim \max\left(1, \frac{\sqrt{{m_G d}}}{m_S} \log\left( \frac{m_S d}{\delta_2} \right)\right).
\end{equation*}
\end{lemma}

\begin{proof}
By the triangle inequality, we have
\begin{align*}
d_{\mathrm{TV}}\left(\sD_G^{m_G}(S), \sD_G^{m_G}(S^i) \right) \leq d_{\mathrm{TV}}\left(\sD_G^{m_G}(S), \sD_G^{m_G}(\Sdrop{i}) \right) + d_{\mathrm{TV}}\left(\Sdrop{i}), \sD_G^{m_G}(S^i) \right).
\end{align*}
In order to bound $d_{\mathrm{TV}}\left(\sD_G^{m_G}(S), \sD_G^{m_G}(S^i) \right)$, We discuss the concentration property of $d_{\mathrm{TV}}\left(\sD_G^{m_G}(S), \sD_G^{m_G}(\Sdrop{i}) \right)$, and the same result will hold for $d_{\mathrm{TV}}\left(\Sdrop{i}), \sD_G^{m_G}(S^i) \right)$. In a similar way as the proof of Lemma~\ref{lemma: kl}, we discuss KL divergence $d_{\mathrm{KL}}\left(\sD_G^{m_G}(S), \sD_G^{m_G}(\Sdrop{i}) \right)$ at first. 

As stated in Lemma~\ref{lemma: gaussian estiamtion}, without loss of generation, we assume that $\epsilon \leq 1/4$, and $m_y$ be the number of samples from category $y$, we have $m_y \ge m/2 - \epsilon m \ge m/4$, $m_y \leq m/2 + \epsilon m \ge 3m/4$, and $\abs{\widehat{\sigma^2_i}/{\sigma^2} - 1} \leq \epsilon$ with probability at least $1- \delta_1$.

In addition, by Lemma~\ref{lemma: gaussian bounded}, given a set $S = \{(\bx_1,y_1), \dots, (\bx_m, y_m)\}$and $\bz_i'$ sampled from the binary mixture Gaussian distribution, with probability at least $1-\delta_2$ we have
\begin{align*}
\max_i \norm{\bx_i}_2 \lesssim \sigma \sqrt{d 
 + \log(\frac{m+1}{\delta_2})}.
\end{align*}

Therefore, by the union bound, the above statements hold with high probability at least $1 - \delta_1 - \delta_2$. We use $\widehat{\sigma}_{k, \setminus i}^2$ to denote the $k$th-dimension variance learned on the set $\Sdrop{i}$, and $\widehat{\mu}_{yk,\setminus i}$ to denote the learned $k$th-dimension mean of the class $y$. We can simplify $d_{\mathrm{KL}}\left(\sD_G^{m_G}(S), \sD_G^{m_G}(\Sdrop{i}) \right)$ as follows,

\begin{align}
&d_{\mathrm{KL}}\left(\sD_G^{m_G}(S), \sD_G^{m_G}(\Sdrop{i}) \right) \nonumber \\
&= m_G d_{\mathrm{KL}}\left(\sD_G(S), \sD_G(\Sdrop{i}) \right) \nonumber \\
&= m_G \int p_G(\bx,y)\log \frac{p_G(\bx,y)}{p_{G^{\setminus i}}(\bx,y)} \nonumber \\
&= m_G \int p_G(\bx,y)\log \frac{p_G(\bx\mid y)p_G(y)}{p_{G^{\setminus i}}(\bx\mid y)p_{G^{\setminus i}}(y)} \nonumber \\
&= m_G \int p_G(\bx,y)\log \frac{p_G(\bx\mid y)}{p_{G^{\setminus i}}(\bx\mid y)} & \text{($p_G(y) = p_{G^{\setminus i}}(y)$)} \nonumber\\
&= m_G \int_y p_G(y) \int_x p_G(\bx\mid y)\log \frac{p_G(\bx\mid y)}{p_{G^{\setminus i}}(\bx\mid y)} \nonumber \\
&= m_G \sum_{y} \frac{1}{2} \int_x p_G(\bx\mid y)\log \frac{p_G(\bx\mid y)}{p_{G^{\setminus i}}(\bx\mid y)} \nonumber \\
&= m_G \sum_{y} \frac{1}{2} \sum_{k=1}^d \frac{1}{2} \left(\frac{\widehat{\sigma}_k^2}{\widehat{\sigma}_{k, \setminus i}^2} - 1  - \log\left( \frac{\widehat{\sigma}_k^2}{\widehat{\sigma}_{k, \setminus i}^2}\right) + \frac{(\widehat{\mu}_{yk} - \widehat{\mu}_{yk,\setminus i})^2}{\widehat{\sigma}_{k, \setminus i}^2}\right) \nonumber \\
&\leq m_G \sum_{y} \frac{1}{2} \sum_{k=1}^d \frac{1}{2} \left( \left(\frac{\widehat{\sigma}_k^2}{\widehat{\sigma}_{k, \setminus i}^2} - 1 \right)^2  + \frac{(\widehat{\mu}_{yk} - \widehat{\mu}_{yk,\setminus i})^2}{\widehat{\sigma}_{k, \setminus i}^2}\right) & \text{($x - log(x + 1) \leq x^2$, $\abs{x} \leq 1/2$)} \nonumber\\
&=m_G\sum_{y} \frac{1}{4}  \left( \sum_{k=1}^d \left(\frac{\widehat{\sigma}_k^2}{\widehat{\sigma}_{k, \setminus i}^2} - 1 \right)^2  + \sum_{k=1}^d \frac{(\widehat{\mu}_{yk} - \widehat{\mu}_{yk,\setminus i})^2}{\widehat{\sigma}_{k, \setminus i}^2}\right). \label{eqn:tau kl}
\end{align}

What we need to bound is $\abs{\widehat{\sigma}_k^2 - \widehat{\sigma}_{k, \setminus i}^2}$ and $\abs{\widehat{\mu}_{yk} - \widehat{\mu}_{yk,\setminus i}}$. They can be bounded by using the boundedness of the data. Without the loss of generation, we assume that $y_i = 0$, then we have

\begin{align*}
\abs{\widehat{\mu}_{0k} - \widehat{\mu}_{0k,\setminus i}} &= \abs{\sum_{j } \frac{x_{jk}}{m_0} - \sum_{j \ne i } \frac{x_{jk}}{m_0 - 1}} \\
&\lesssim \abs{\sum_{j } \frac{x_{jk}}{m_0} - \sum_{j \ne i } \frac{x_{jk}}{m_0}} \\
&\lesssim \abs{\frac{x_{jk}}{m_0}} \lesssim \abs{\frac{x_{jk}}{m}} \\
&\leq \frac{1}{m}\abs{ \mu_{0k} + \sqrt{2} \sigma \sqrt{\log\left( \frac{(m+1)d}{\delta_2} \right)} } \\
&\lesssim \frac{1}{m} \sigma \sqrt{\log\left( \frac{(m+1)d}{\delta_2} \right)},\\
\abs{\widehat{\mu}_{1k} - \widehat{\mu}_{1k,\setminus i}} = 0.
\end{align*}

Therefore, we have

\begin{align}
\sum_{i=1}^d \frac{(\widehat{\mu_{0k}} - \widehat{\mu}_{0k,\setminus i})^2}{\widehat{\sigma}_{k, \setminus i}^2} &\lesssim \sum_{k=1}^d \frac{1}{\widehat{\sigma}_{k, \setminus i}^2} \frac{1}{m^2} \sigma^2 \log\left( \frac{(m+1)d}{\delta_2} \right) \nonumber \\
&\lesssim \sum_{k=1}^d  \frac{1}{m^2} \log\left( \frac{(m+1)d}{\delta_2} \right) \lesssim \frac{d}{m^2} \log\left( \frac{(m+1)d}{\delta_2} \right), \label{eqn:mu0k}\\
\sum_{i=1}^d \frac{(\widehat{\mu_{1k}} - \widehat{\mu}_{1k,\setminus i})^2}{\widehat{\sigma}_{k, \setminus i}^2} = 0. \label{eqn:mu1k}
\end{align}

In terms of $\abs{\widehat{\sigma}_k^2 - \widehat{\sigma}_{k, \setminus i}^2}$, we can write

\begin{align*}
\abs{\widehat{\sigma}_k^2 - \widehat{\sigma}_{k, \setminus i}^2} &= \abs{\frac{m_0}{m}\frac{\sum_j (x_{jk} - \widehat{\mu}_{0k})^2}{m_0 - 1} - \frac{m_0-1}{m}\frac{\sum_{j \ne i} (x_{jk} - \widehat{\mu}_{0k,\setminus i})^2)^2}{m_0 - 2}}\\
& \lesssim \abs{\frac{m_0}{m}\frac{\sum_j (x_{jk} - \widehat{\mu}_{0k})^2}{m_0 - 1} - \frac{m_0}{m}\frac{\sum_{j \ne i} (x_{jk} - \widehat{\mu}_{0k,\setminus i})^2}{m_0 - 1}}\\
& = \frac{m_0}{m(m_0 - 1)} \abs{x_{ik}^2 + (m_0 - 1)\widehat{\mu}_{0k,\setminus i}^2 - m_0\widehat{\mu}_{0k}^2} \\
& \lesssim \frac{m_0}{m(m_0 - 1)} \abs{x_{ik}^2 + m_0 \widehat{\mu}_{0k,\setminus i}^2 - m_0\widehat{\mu}_{0k}^2} \\
& \lesssim \frac{m_0}{m(m_0 - 1)} \abs{x_{ik}^2 + m_0 \left( \widehat{\mu}_{0k,\setminus i}^2 - \widehat{\mu}_{0k}^2 \right)} \\
& = \frac{m_0}{m(m_0 - 1)} \abs{x_{ik}^2 + m_0 \left( \widehat{\mu}_{0k,\setminus i} - \widehat{\mu}_{0k} \right) \left( \widehat{\mu}_{0k,\setminus i} + \widehat{\mu}_{0k} \right)} \\
& = \frac{m_0}{m(m_0 - 1)} \abs{x_{ik}^2 + m_0 \left( \sum_{j \ne i } \frac{x_{jk}}{m_0 - 1} - \sum_{j } \frac{x_{jk}}{m_0}  \right) \left( \sum_{j \ne i } \frac{x_{jk}}{m_0 - 1} + \sum_{j } \frac{x_{jk}}{m_0} \right)} \\
& \lesssim \frac{m_0}{m(m_0 - 1)} \abs{x_{ik}^2 + m_0 \frac{x_{ik}}{m_0} \left( \sum_{j \ne i } \frac{x_{jk}}{m_0 - 1} + \sum_{j } \frac{x_{jk}}{m_0} \right)} \\
& \lesssim \frac{1}{m} \left( \abs{x_{ik}^2} + \abs{ {x_{ik}} \left( \sum_{j \ne i } \frac{x_{jk}}{m_0 - 1} + \sum_{j } \frac{x_{jk}}{m_0} \right)}  \right)\\
& \lesssim  \frac{1}{m} \left( \sigma^2 \log\left( \frac{(m+1)d}{\delta_2} \right) + 2 \sigma^2 \log\left( \frac{(m+1)d}{\delta_2} \right) \right)\\
& \lesssim  \frac{\sigma^2}{m} \log\left( \frac{(m+1)d}{\delta_2} \right).
\end{align*}
Thus, we can obtain

\begin{align}
\sum_{k=1}^d \left(\frac{\widehat{\sigma}_k^2}{\widehat{\sigma}_{k, \setminus i}^2} - 1 \right)^2 &\lesssim \sum_{k = 1}^d \frac{\sigma^4}{\widehat{\sigma}_{k, \setminus i}^4 m^2} \log^2\left( \frac{(m+1)d}{\delta_2} \right) \nonumber \\
&\lesssim \frac{d}{m^2} \log^2\left( \frac{(m+1)d}{\delta_2} \right). \label{eqn:sigmai}
\end{align}

By plugin (\ref{eqn:mu0k}), (\ref{eqn:mu1k}) and (\ref{eqn:sigmai}) into (\ref{eqn:tau kl}), we have

\begin{align*}
d_{\mathrm{KL}}\left(\sD_G^{m_G}(S), \sD_G^{m_G}(\Sdrop{i}) \right) &\leq m_G \sum_{y} \frac{1}{4}  \left( \sum_{k=1}^d \left(\frac{\widehat{\sigma}_k^2}{\widehat{\sigma}_{k, \setminus i}^2} - 1 \right)^2  + \sum_{k=1}^d \frac{(\widehat{\mu}_{yk} - \widehat{\mu}_{yk,\setminus i})^2}{\widehat{\sigma}_{k, \setminus i}^2}\right) \\
& \lesssim  m_G \frac{d}{m^2} \log^2\left( \frac{(m+1)d}{\delta_2} \right),
\end{align*}
which implies

\begin{align*}
d_{\mathrm{TV}}\left(\sD_G^{m_G}(S), \sD_G^{m_G}(\Sdrop{i}) \right) &\lesssim \max\left(2, \sqrt{d_{\mathrm{KL}}\left(\sD_G^{m_G}(S), \sD_G^{m_G}(\Sdrop{i}) \right)} \right) \\
&\lesssim \max\left(2, \sqrt{d_{\mathrm{KL}}\left(\sD_G^{m_G}(S), \sD_G^{m_G}(\Sdrop{i}) \right)}\right) \\
&\lesssim \max\left(2, \frac{\sqrt{{m_G d}}}{m} \log\left( \frac{(m+1)d}{\delta_2} \right)\right),
\end{align*}
and

\begin{align*}
d_{\mathrm{TV}}\left(\sD_G^{m_G}(S), \sD_G^{m_G}(S^i) \right) &\leq d_{\mathrm{TV}}\left(\sD_G^{m_G}(S), \sD_G^{m_G}(\Sdrop{i}) \right) + d_{\mathrm{TV}}\left(\Sdrop{i}), \sD_G^{m_G}(S^i) \right)\\
&\lesssim \max\left(1, \frac{\sqrt{{m_G d}}}{m} \log\left( \frac{(m+1)d}{\delta_2} \right)\right)\\
&\lesssim \max\left(1, \frac{\sqrt{{m_G d}}}{m} \log\left( \frac{md}{\delta_2} \right)\right).
\end{align*}
Because it holds for all $i$, the proof of Lemma~\ref{lemma: tau} is completed.
\end{proof}

Now we are ready to prove Theorem~\ref{thm: bGMM generalization bound}.
\begin{proof}

Let $\delta$ in Lemma~\ref{lemma: kl} be $\delta_1$ and that in Lemma~\ref{lemma: gaussian bounded} be $\delta_2$. With probability at least $1-\delta / 2$, the bounds in  Lemma~\ref{lemma: kl} hold with $\delta_1 = \delta / 2$.Then with probability at least $1-\delta / 2$. Besides, the bounds in  Lemma~\ref{lemma: M} and Lemma~\ref{lemma: beta} hold with $\delta_2 = \delta / 2$. Thus, by the union bound, we know that with high probability $1-\delta$, the above bounds hold. Furthermore, from the proof of Lemma~\ref{lemma: tau}, we know that it holds naturally in this case, where the boundedness of data points and the accurate estimation of the true distribution hold.

Finally, we plugin Lemma~\ref{lemma: M},~\ref{lemma: kl},~\ref{lemma: beta},~\ref{lemma: tau} into Theorem~\ref{thm: main generalization bound}, and can conclude the statement of Theorem~\ref{thm: bGMM generalization bound} with high probability at least $1-\delta$,

\begin{align}
&\abs{\textit{Gen-error}} \nonumber \\
&\lesssim \frac{m_G}{m_T} \left(d + \log \left(\frac{m_T}{\delta} \right)\right) \max\left(1, \sqrt{\frac{d}{m_S}\log\left(\frac{d}{\delta} \right)}\right) \nonumber \\
&\quad + \frac{\sqrt{m_S} + \sqrt{m_G} }{m_T}  \left(d + \log \left(\frac{m_T}{\delta} \right)\right) \sqrt{\log \left(\frac{1}{\delta}\right)}  +  \frac{ m_S\sqrt{m_G} }{m_T^2} \left(d + \log \left(\frac{m_T}{\delta} \right)\right) \sqrt{\log \left(\frac{1}{\delta}\right)} \nonumber \\
&\quad +\frac{m_S \log m_S + m_G \log m_G}{m_T^2} \left(d + \log \left(\frac{m_T}{\delta} \right)\right) \log \left(\frac{1}{\delta}\right) \nonumber \\
&\quad +\frac{ m_S \log m_S }{m_T} \left(d + \log \left(\frac{m_T}{\delta} \right)\right) \max\left(1, \frac{\sqrt{{m_G d}}}{m_S} \log\left( \frac{m_S d}{\delta} \right)\right) \log \left(\frac{1}{\delta}\right) \label{eqn: upper bound bgmm}\\
&\lesssim \begin{cases} 
\frac{\log(m_S)}{\sqrt{m_S}} & \text{ if fix $d$ and $m_G = 0$,}\\
\frac{\log^2(m_S)}{\sqrt{m_S}} & \text{ if fix $d$ and $m_G = \Theta(m_S)$,}\\
\frac{\log(m_S)}{\sqrt{m_S}} & \text{ if fix $d$ and $m_G = m_{G, \mathrm{order}}^*$,}\\
d & \text{ if fix $m_S$.} \nonumber
\end{cases}
\end{align}

\end{proof}

\subsection{Proof of Theorem~\ref{thm: GAN generalization bound}}
\label{sec: proof of GAN generalization bound}
The theorem is built upon the recent theoretical works on GAN~\cite{DBLP:journals/jmlr/Liang21} and SGD~\cite{DBLP:conf/uai/ZhangZBP0022, DBLP:journals/corr/mingzewang}. We first list some lemmas from these works.

\begin{lemma}[Upper bounds for output and gradient, Proposition 5.2,~\cite{DBLP:journals/corr/mingzewang}]
\label{lemma: Upper bounds for output and gradient}
For deep CNNs or MLPs in Appendix~\ref{sec: deep classifier arch}, we have

\begin{align*}
& |f(\bw, \bx)| \leq  \left( \prod_{l=1}^{L}\left\|\bw_l\right\|_2 \right) \Vert \bx\Vert_2, \\
& \left\|\frac{\partial f(\bw, \bx)}{\partial \bw_l}\right\|_2 \leq \left( \prod_{i \neq l}\left\| \bw_l \right\|_2\right) \Vert \bx \Vert_2.
\end{align*}
\end{lemma}

\begin{lemma}[Uniform stability of SGD in the non-convex case, Theorem 5,~\cite{DBLP:conf/uai/ZhangZBP0022}]
\label{lemma: SGD stability}

Assume $f$ is $\beta$-smooth and $\rho$-Lipschitz. Running $T>m$ iterations of SGD with step size $\alpha_t=\frac{c}{\beta t}$, the stability of SGD satisfies
$$
\beta_m \leq \frac{16 \rho^2 T^c}{m^{1+c}}.
$$

\end{lemma}

\begin{lemma}[Learnability of GAN, Theorem 19,~\cite{DBLP:journals/jmlr/Liang21}]
\label{lemma: Learnability of GAN}

We suppose that the architecture of GAN is the same as that in Appendix~\ref{sec: GAN arch}. Besides, we consider the realizable setting, that is, $\sD$ enjoys the same distribution as $g_{\theta_*}(Z)$ with some $\theta_* \in \Theta(d, L)$ and $Z \sim \mathrm{unif}[0,1]^d$. Then, given training set $S$ with $m$ i.i.d. samples, it holds that
$$
\mathbb{E} d_{\mathrm{TV}}^2 \left(\sD, \DG \right) \lesssim \sqrt{d^2 L^2 \log (d L)\frac{\log m}{m} } .
$$
\end{lemma}

\begin{proof}

Now we are ready to prove Theorem~\ref{thm: GAN generalization bound}, the main idea is to bound $M$, $\beta_{m_T}$, $d_{\mathrm{TV}} \left(\sD, \DG \right)$ in Theorem~\ref{thm: main generalization bound}. $M$ and Lipschitz property can be bounded by using Lemma~\ref{lemma: Upper bounds for output and gradient}. $\beta_{m_T}$ can be induced by Lemma~\ref{lemma: SGD stability} with Lipschitz constant. In terms of $d_{\mathrm{TV}} \left(\sD, \DG \right)$, Lemma~\ref{lemma: Learnability of GAN} can be used to derive an upper bound.

First, we bound the loss function as follows.
\begin{align*}
\ell(f,  \bz)& = \ell(f, (\bx, y)) \\
&= \log(1 + \exp(-y f(\bw, \bx))) \\
&\leq \log(2) + \abs{y f(\bw, \bx)} & \text{($\log(1 + \exp(-t))$ is 1-Lipschitz)}\\
&= \log(2) + \abs{f(\bw, \bx)} \\
&\leq \log(2) + \left( \prod_{l=1}^{L}\left\|\bw_l\right\|_2 \right) \Vert \bx\Vert_2 & \text{(by Lemma~\ref{lemma: Upper bounds for output and gradient})}\\
&\leq \log(2) + \left( \prod_{l=1}^{L}\left\|\bw_l\right\|_2 \right) \sqrt{d} \\
&\leq \log(2) + \left( \prod_{l=1}^{L}\left\|W_l\right\|_2 \right) \sqrt{d} \\
&\lesssim \left( \prod_{l=1}^{L}\left\|W_l\right\|_2 \right) \sqrt{d} .
\end{align*}
Thus, we have $M \lesssim \left( \prod_{l=1}^{L}\left\|W_l\right\|_2 \right) \sqrt{d}$.

Second, we prove that $f$ is Lipschitz given the bounded parameter space.

\begin{align*}
\left\|\frac{\partial f(\bw, \bx)}{\partial \bw}\right\|_2 &\leq \sum_{l=1}^L \left\|\frac{\partial f(\bw, \bx)}{\partial \bw_l}\right\|_2 \\
&\leq \Vert \bx \Vert_2 \sum_{l=1}^L \left( \prod_{i \neq l}\left\| \bw_l \right\|_2\right) & \text{(by Lemma~\ref{lemma: Upper bounds for output and gradient})}\\
&\leq \Vert \bx \Vert_2 \sum_{l=1}^L \left( \prod_{i \neq l}\left\| W_l \right\|_2\right) \\
&\leq \sqrt{d} \sum_{l=1}^L \left( \prod_{i \neq l}\left\| \bw_l \right\|_2\right) \\
&\lesssim \sqrt{d} \sum_{l=1}^L \left( \prod_{i}\left\| W_l \right\|_2\right) \\
&= \sqrt{d} L \left( \prod_{i}\left\| W_l \right\|_2\right).
\end{align*}
Therefore, $f$ is $\rho$-Lipschitz with $\sqrt{d} L \left( \prod_{i}\left\| W_l \right\|_2\right)$. Then, $\beta_{m}$ can be bounded by Lemma~\ref{lemma: SGD stability}.

\begin{align*}
\beta_m \leq \frac{16 \rho^2 T^c}{m^{1+c}} \leq 16 {d} L^2 \left(\prod_{i }\left\| W_l \right\|_2 \right)^2  \frac{T^c}{m^{1+c}} \lesssim \left(  \prod_{i}\left\| W_l \right\|_2 \right)^2  \frac{{d} L^2 }{m}.
\end{align*}

Third, we bound the expectation of divergence between model distribution and target distribution as follows.

\begin{align*}
 \mathbb{E} d_{\mathrm{TV}} \left(\sD, \DG \right)  &\lesssim  \E \int_{(\bx, y)} \abs{\P_{\sD}(\bx, y) - \P_{\DG}(\bx, y)} d\bz \\
 &= \E \sum_y \int_{\bx} \abs{\P_{\sD}(\bx, y) - \P_{\DG}(\bx, y)} d\bx \\
 &= \E \sum_y \int_{\bx} \frac{1}{2} \abs{\P_{\sD}(\bx\mid y) - \P_{\DG}(\bx\mid y)} d\bx \\
  &= \E \sum_y d_{\mathrm{TV}} \left(\P_{\sD}(\bx\mid y) , \P_{\DG}(\bx\mid y) \right) \\
&= \sum_y \E d_{\mathrm{TV}} \left(\P_{\sD}(\bx\mid y) , \P_{\DG}(\bx\mid y) \right) \\
&=\sum_y \sqrt{ \left( \E d_{\mathrm{TV}} \left(\P_{\sD}(\bx\mid y) , \P_{\DG}(\bx\mid y) \right) \right)^2} \\
&\leq \sum_y \sqrt{  \E d_{\mathrm{TV}}^2 \left(\P_{\sD}(\bx\mid y) , \P_{\DG}(\bx\mid y) \right) } \\
&\lesssim \sum_y \sqrt{ \sqrt{d^2 L^2 \log (d L)\frac{\log m}{m} }} \\
&\lesssim \left(d^2 L^2 \log (d L)\frac{\log m}{m} \right)^{\frac{1}{4}} = \sqrt{dL} \left( \log (d L)\frac{\log m}{m} \right)^{\frac{1}{4}}.
\end{align*}
Furthermore, because $d_{\mathrm{TV}} \left(\sD, \DG \right) \leq 1$, we have

\begin{align*}
\mathbb{E} d_{\mathrm{TV}} \left(\sD, \DG \right) \lesssim \max\left(1,  \sqrt{dL} \left( \log (d L)\frac{\log m}{m} \right)^{\frac{1}{4}}\right).
\end{align*}

Finally, by taking the expectation for the bound in Theorem~\ref{thm: main generalization bound}, and plugging $M$, $\beta_{m_T}$ and $d_{\mathrm{TV}} \left(\sD, \DG \right)$ into it, we can conclude the result of Theorem~\ref{thm: GAN generalization bound}.
\end{proof}

\section{Discussion on existing non-i.i.d. stability bounds}
\label{sec: Discussion for existing non-i.i.d. stability bounds}

In this section, we show that it is unclear how to use existing non-i.i.d. stability bounds to derive a better guarantee than Theorem~\ref{thm: main generalization bound} for GDA.

\subsection{Stability bounds for mixing processes}

To the best of our knowledge, existing stability bounds for mixing processes only focus on the stationary sequence~\cite{DBLP:conf/nips/MohriR07,DBLP:journals/jmlr/MohriR10,DBLP:journals/ijon/HeZC16}, which is defined as follows.

\begin{mydef}[Stationary sequence]
A sequence of random variables $\mathbf{Z}=\left\{Z_t\right\}_{t=-\infty}^{\infty}$ is said to be stationary if for any $t$ and non-negative integers $m$ and $k$, the random vectors $\left(Z_t, \ldots, Z_{t+m}\right)$ and $\left(Z_{t+k}, \ldots, Z_{t+m+k}\right)$ have the same distribution.
\end{mydef}

Unfortunately, the GDA setting in this paper does not satisfy the stationary condition, because $(\bz_1, \dots, \bz_{m_S}) = S$ and $(\bz_{m_S + 1}, \dots, \bz_{2 m_S}) \subseteq S_G$ do not have the same distribution. Furthermore, it is usually difficult to estimate the mixing coefficients which reflect quantitative dependencies among data points.

\begin{figure}[t]
\centering

\includegraphics[width=0.8\columnwidth]{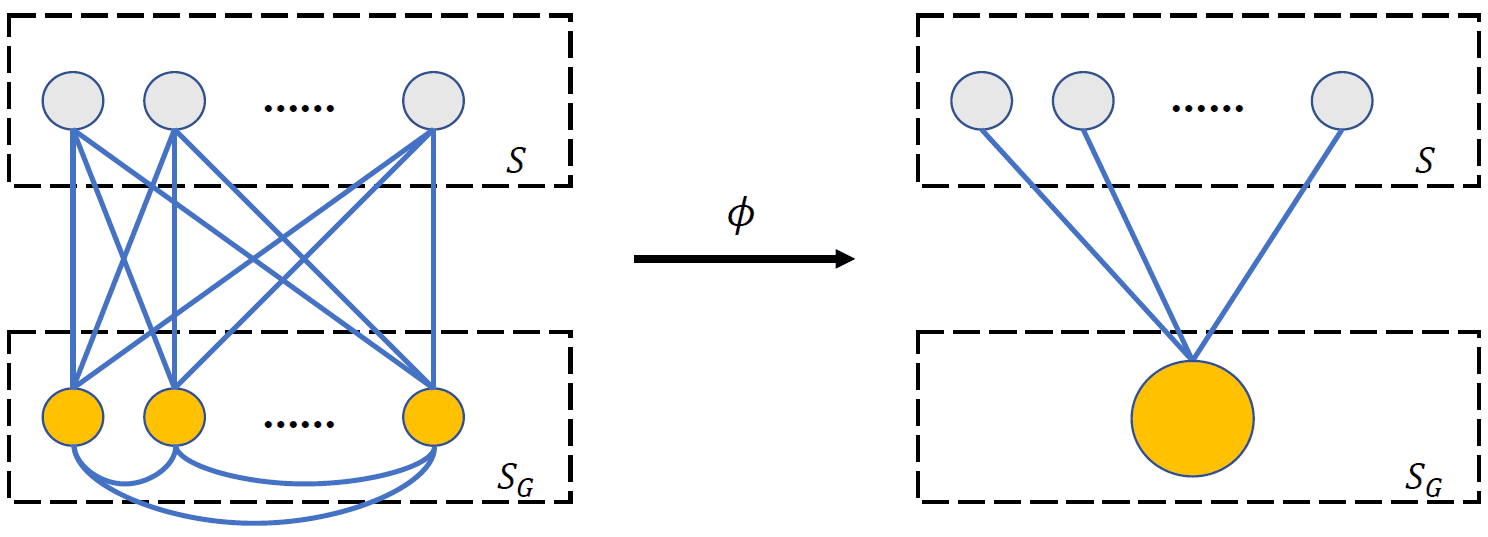}

\vskip 0.15in

\caption{Dependence graph (left) and a forest approximation (right) of the GDA setting.}
\label{fig: forest}
\end{figure}

\subsection{Stability bounds for dependence graph}

Recently,~\cite{DBLP:conf/nips/ZhangL0W19} provide a framework for the generalization theory of graph-dependent data, which includes the classical stability result in \cite{DBLP:journals/jmlr/BousquetE02} as a special case. We now introduce some elements of graph-dependent random variables and the non-i.i.d. stability bound in~\cite{DBLP:conf/nips/ZhangL0W19}. For a graph $G$, we use $V(G)$ to denote its vertex set and $E(G)$ to denote its edge set.

\begin{mydef}[Dependency Graph, Definition 3.1 in~\cite{DBLP:conf/nips/ZhangL0W19}]
An undirected graph $G$ is called a dependency graph of a random vector $\bX = (X_1, \dots ,X_n)$ if (1) $V(G) = [n]$, (2) if $I, J \subseteq [n]$ are non-adjacent in $G$, then $\{X_i\}_{i\in I}$ and $\{X_j\}_{j\in J}$ are independent.
\end{mydef}

\begin{mydef}[Forest Approximation, Definition 3.4 in~\cite{DBLP:conf/nips/ZhangL0W19}]
Given a graph $G$, a forest $F$, and a mapping $\phi : V(G) \to V(F)$, if $\phi(u) = \phi(v)$ or edge $\tri{\phi(u), \phi(v)} \in E(F)$ for any edge $\tri{\phi(u), \phi(v)} \in E(G)$, then $(\phi, F)$ is called a forest approximation of $G$. Let $\Phi(G)$ be the set of forest approximations of $G$. 
\end{mydef}

\begin{mydef}[Forest Complexity, Definition 3.5 in~\cite{DBLP:conf/nips/ZhangL0W19}]
Given a graph $G$ and any forest approximation $(\phi, F) \in \Phi(G)$ with $F$ consisting of trees $\left\{T_i\right\}_{i \in[k]}$, let
$$
\lambda_{(\phi, F)}=\sum_{\langle u, v\rangle \in E(F)}\left(\left|\phi^{-1}(u)\right|+\left|\phi^{-1}(v)\right|\right)^2+\sum_{i=1}^k \min _{u \in V\left(T_i\right)}\left|\phi^{-1}(u)\right|^2 .
$$
We call $\Lambda(G)=\min _{(\phi, F) \in \Phi(G)} \lambda_{(\phi, F)}$ the forest complexity of the graph $G$.
\end{mydef}

\begin{theorem}
\label{thm: graph dependence bound}
Assume that $\sA$ is a $\beta_m$-stable. Given a set $\Saug$ of size $m$ sampled from the same marginal distribution $\sD$ with dependency graph $G$. Suppose the maximum degree of $G$ is $\Delta$, and the loss function $\ell$ is bounded by $M$. For any $\delta \in (0,1)$, with probability at least $1-\delta$, it holds that

\begin{equation*}
\sR_{\sD}(\sA(\Saug)) \leq \widehat{\sR}_{\Saug}(\sA(\Saug)) + 2\beta_{m, \Delta}(\Delta+1) + (4\beta_m + \frac{M}{m})\sqrt{\frac{\Lambda(G)}{2} \log(\frac{1}{\delta})},
\end{equation*}
where $\beta_{m, \Delta} = \max_{i \leq \Delta} \beta_{m-i}$ and $\Lambda(G)$ is the forest complexity of the dependence graph $G$.
\end{theorem}

\begin{remark}
Theorem~\ref{thm: graph dependence bound} requires $\Saug$ sampled from the same marginal distribution $\sD$, which fails to hold in the context of GDA because the learned distribution $\DG$ is generally not the same as the true distribution $\sD$. It is still unclear to overcome this problem.
\end{remark}

\begin{remark}
When $m_G = 0$ and $\Saug = S$, Theorem~\ref{thm: graph dependence bound} degenerates to the classical result in~\cite{DBLP:journals/jmlr/BousquetE02}, which requires $\beta_m = o(1 / \sqrt{m})$ to converge. In contrast, Theorem~\ref{thm: main generalization bound} only requires $\beta_m = o(1 / \log(m))$ to converge, which is better than that of Theorem~\ref{thm: graph dependence bound}.
\end{remark}

\begin{remark}
We note that Theorem~\ref{thm: graph dependence bound} is proposed for the general case with data dependence. Therefore, it does not consider the property of special cases and may fail to give good guarantees. On the one hand, the independence of $S$ and the conditional independence of $S_G$ used in the proof of Theorem~\ref{thm: main generalization bound} are significant, which is ignored by Theorem~\ref{thm: graph dependence bound}. On the other hand, in the case of strong dependence like GDA, the forest complexity may be too large to give a meaningful bound. The dependence graph and a forest approximation of the GDA setting are presented in Figure~\ref{fig: forest}. Therefore, the forest complexity of the GDA setting can be bounded as follows.

\begin{equation}
\Lambda(G) \leq m_S (1 + m_G)^2 + 1^2 \lesssim m_Sm_G^2. \label{eqn: forest complexity}
\end{equation}

Plugging (\ref{eqn: forest complexity}) into Theorem~\ref{thm: graph dependence bound}, and assume $m_G = \Theta(m_S)$, we observe that

\begin{equation*}
\frac{M}{m_T}\sqrt{\frac{\Lambda(G)}{2} \log(\frac{1}{\delta})} \lesssim \frac{M}{m_T}\sqrt{\frac{m_Sm_G^2}{2} \log(\frac{1}{\delta})} \lesssim M \sqrt{\frac{m_S}{2} \log(\frac{1}{\delta})},
\end{equation*}
which fails to converge. However, Theorem~\ref{thm: main generalization bound} overcomes this problem.

\end{remark}

Finally, we conclude that it is hard to directly use existing non-i.i.d. stability results to obtain a better guarantee than Theorem~\ref{thm: main generalization bound}.

\section{Experimental details and additional results}
\label{sec: Additional experimental details and results}

All details can be found in the attached code.

\subsection{Models}

\textbf{bGMM}. We adopt the implementation of na\"ive Bayes in~\cite{DBLP:journals/corr/chenyu} to estimate the parameters of bGMM.

\textbf{ResNet.} We add the ResNet50 checkpoint released by Pytorch~\cite{pytorch}, which is also used in~\cite{DBLP:journals/corr/abs-2302-04638}.

\textbf{cDCGAN.} We use the cDCGAN in this \href{https://github.com/znxlwm/pytorch-MNIST-CelebA-cGAN-cDCGAN}{repository}, and modify its input channel and label dimension to 3 and 10 respectively to keep consistent with the format of images in CIFAR-10 dataset. This repository gains the most stars among repositories that implement cDCGAN. Furthermore, we follow its hyperparameter setting and train 200 epochs to obtain a cDCGAN for the CIFAR-10 dataset.

\textbf{StyleGAN2-ADA.} We use the class-conditional model pre-trained on CIFAR-10 dataset, which is released by NVIDIA Research~\cite{DBLP:conf/nips/KarrasAHLLA20}.

\textbf{EDM.} We use the 5M synthetic CIFAR-10 dataset released in~\cite{DBLP:journals/corr/abs-2302-04638}, which is generated by the pre-trained conditional EDM. Given an augmentation size $m_G$, we randomly sample $m_G$ from the 5M synthetic data points.

\subsection{Training details}

\textbf{Standard data augmentation.} 4 pixels are padded on each side, and a $32 \times 32$ crop is randomly sampled from the padded image or its horizontal flip. This augmentation pipeline is widely used~\cite{DBLP:conf/cvpr/HeZRS16}.

\textbf{Optimization.} We follow the setting in~\cite{DBLP:journals/corr/abs-2302-04638}. We use the SGD optimizer, where the momentum and weight decay are set to 0.9 and $5 \times 10^{-4}$, respectively. We
use the cyclic learning rate schedule with cosine annealing, where the initial learning rate is set to 0.2. We train the deep neural classifier with 100 epochs. The batch size is 512.

\subsection{Computation consumption.} All experiments are run on one RTX 3090 GPU. The most consuming case (ResNet50, $m_G =$ 1M) takes 17 GB cuda memory and 20 hours.

\subsection{License}

The used codes and their licenses are listed in Table~\ref{tab: license}.

\begin{table}[h]
\centering
\caption{The used codes and licenses.}
\vskip 0.15in
\begin{tabular}{ccc} 
\toprule
URL                                                    & Citation & License                                                    \\ 
\midrule
https://github.com/NVlabs/stylegan2-ada-pytorch & \cite{DBLP:conf/nips/KarrasAHLLA20}       & \href{https://github.com/NVlabs/stylegan2-ada-pytorch/blob/main/LICENSE.txt}{License}                                         \\
https://github.com/pytorch/pytorch                     & \cite{pytorch}   &  \href{https://github.com/pytorch/pytorch/blob/master/LICENSE}{License}      \\
https://github.com/wzekai99/DM-Improves-AT               & \cite{DBLP:journals/corr/abs-2302-04638}      & MIT License                                                \\
https://github.com/ML-GSAI/Revisiting-Dis-vs-Gen-Classifiers             & \cite{DBLP:journals/corr/chenyu}    & MIT License                                 \\
https://github.com/znxlwm/pytorch-MNIST-CelebA-cGAN-cDCGAN                         & -      & -                                                \\
\bottomrule
\end{tabular}
\label{tab: license}
\end{table}

\subsection{Additional results}
\label{sec: Additional results}

Our empirical results on the CIFAR-10 dataset are presented in Table~\ref{tab: deep result}.

\begin{table}[h]
\centering
\caption{Accuracy on the CIFAR-10 test set, where S.A. denotes standard augmentation.}
\vskip 0.15in
\begin{tblr}{
  cells = {c},
  cell{1}{1} = {r=2}{},
  cell{1}{2} = {r=2}{},
  cell{1}{3} = {r=2}{},
  cell{1}{4} = {c=6}{},
  cell{3}{1} = {r=6}{},
  cell{3}{2} = {r=2}{},
  cell{5}{2} = {r=2}{},
  cell{7}{2} = {r=2}{},
  cell{9}{1} = {r=6}{},
  cell{9}{2} = {r=2}{},
  cell{11}{2} = {r=2}{},
  cell{13}{2} = {r=2}{},
  cell{15}{1} = {r=6}{},
  cell{15}{2} = {r=2}{},
  cell{17}{2} = {r=2}{},
  cell{19}{2} = {r=2}{},
  hline{1,21} = {-}{0.08em},
  hline{2} = {4-9}{},
  hline{3,9,15} = {-}{},
  hline{5,7,11,13,17,19} = {2-9}{},
}
Generator     & Classifier & S.A.     & GDA ($m_G$) &       &       &       &       &       \\
              &            &                           & 0                     & 100k  & 300k  & 500k  & 700k  & 1M    \\
cDCGAN~\cite{DBLP:journals/corr/RadfordMC15}        & ResNet18   & $\times$ & 85.76                   & 86.8  & 87.83 & 87.59 & 87.52 & 86.47 \\
              &            & $\surd$  & 94.4                    & 93.92 & 93.41 & 93.81 & 93.01 & 92.6  \\
              & ResNet34   & $\times$ & 85                      & 86.9  & 87.93 & 87.56 & 87.17 & 86.28 \\
              &            & $\surd$  & 94.59                   & 94.83 & 94.21 & 93.64 & 93.69 & 93.18 \\
              & ResNet50   & $\times$ & 82.85                   & 87.49 & 88.59 & 86.67 & 86.3  & 85.2  \\
              &            & $\surd$  & 94.69                   & 94.43 & 93.86 & 93.74 & 93.12 & 92.63 \\
StyleGAN2-ADA~\cite{DBLP:conf/nips/KarrasAHLLA20} & ResNet18   & $\times$ & 85.76                   & 90.22 & 91.33 & 91.37 & 91.25 & 91.38 \\
              &            & $\surd$  & 94.4                    & 94.68 & 94.46 & 94.4  & 94.11 & 94.12 \\
              & ResNet34   & $\times$ & 85                      & 90.24 & 91.23 & 91.45 & 91.56 & 90.91 \\
              &            & $\surd$  & 94.59                   & 95.05 & 94.9  & 94.4  & 94.43 & 94.21 \\
              & ResNet50   & $\times$ & 82.85                   & 90.85 & 92.29 & 92.29 & 92.29 & 91.61 \\
              &            & $\surd$  & 94.69                   & 94.74 & 95.04 & 94.56 & 94.76 & 94.28 \\
EDM~\cite{DBLP:journals/corr/EDM}           & ResNet18   & $\times$ & 85.76                   & 92.8  & 94.87 & 95.43 & 96.24 & 96.28 \\
              &            & $\surd$  & 94.4                    & 96.15 & 96.74 & 97.09 & 97.28 & 97.5  \\
              & ResNet34   & $\times$ & 85                      & 93.42 & 94.93 & 95.59 & 96.14 & 96.44 \\
              &            & $\surd$  & 94.59                   & 96.47 & 96.96 & 97.36 & 97.53 & 97.51 \\
              & ResNet50   & $\times$ & 82.85                   & 93.29 & 95.29 & 95.95 & 96.1  & 96.64 \\
              &            & $\surd$  & 94.69                   & 96.09 & 96.87 & 97.28 & 97.6  & 97.74 
\end{tblr}
\label{tab: deep result}
\end{table}

\end{appendices}

\end{document}